\documentclass{article} 
\usepackage[preprint]{colm2026_conference}
\usepackage{booktabs}
\usepackage{pifont}

\usepackage{graphicx}  
\newcommand{\methodname}{\textsc{SDAR}\xspace}
 \usepackage{amsthm}
\newtheorem{proposition}{Proposition} 
\usepackage{microtype}
\usepackage{hyperref}
\usepackage{algorithm}
\usepackage{algpseudocode}
\usepackage{url}
\usepackage{booktabs}
\usepackage{xcolor}  
\usepackage{colortbl}
\usepackage{xspace}
\usepackage{soul}       
\newcolumntype{C}{>{\centering\arraybackslash}p{0.038\textwidth}}
\definecolor{topcolor}{RGB}{252, 236, 196}
\definecolor{secondcolor}{RGB}{223, 235, 253}
\usepackage[most,skins,theorems]{tcolorbox}
\definecolor{darkgreen}{RGB}{0,128,0}
\newcommand{\mycomment}[1]{\textcolor{darkgreen}{\textit{// #1}}}
\newcommand{\posval}[1]{{\textcolor{darkgreen}{#1}}}

\newtcolorbox{templatebox}[1]{
  enhanced,
  unbreakable,
  colback=white,
  colframe=black!65,
  colbacktitle=black!80,
  coltitle=white,
  boxrule=0.9pt,
  arc=2pt,
  left=6pt,
  right=6pt,
  top=6pt,
  bottom=6pt,
  title={#1},
  fonttitle=\bfseries,
  sharp corners,
  boxed title style={sharp corners, boxrule=0pt}
}

\usepackage{lineno}

\definecolor{darkblue}{rgb}{0, 0, 0.5}
\hypersetup{colorlinks=true, citecolor=darkblue, linkcolor=darkblue, urlcolor=darkblue}
\AddToShipoutPictureBG{%
  \ifnum\value{page}=1
    \AtPageUpperLeft{%
      \put(\LenToUnit{3.8cm},\LenToUnit{-2cm}){%
        \includegraphics[height=0.8cm]{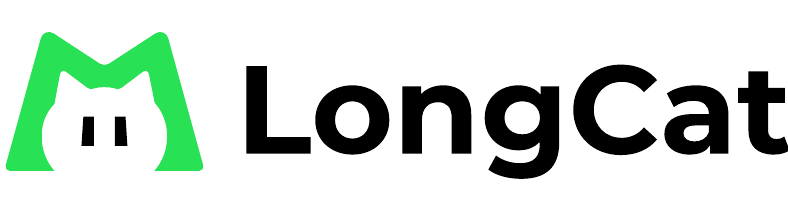}%
      }%
    }%
  \fi
}

\title{Self-Distilled Agentic Reinforcement Learning}

\author{
\textbf{Zhengxi Lu}$^{1,2}$\thanks{Work done during internship at Meituan.},~ \textbf{Zhiyuan Yao}$^{1,2}$, 
\textbf{Zhuowen Han}$^{2}$, \textbf{Zi-Han Wang}$^{2,3}$, \textbf{Jinyang Wu}$^{3}$\\
\textbf{Qi Gu}$^{2}$\thanks{Corresponding author},~
\textbf{Xunliang Cai}$^{2}$,
\textbf{Weiming Lu}$^{1}$, 
\textbf{Jun Xiao}$^{1}$,
\textbf{Yueting Zhuang}$^{1}$,
 \textbf{Yongliang Shen}$^{1}$\footnotemark[2]\\[3pt]
  $^1$Zhejiang University \qquad$^2$Meituan \qquad$^3$Tsinghua University\\
  \texttt{\small \{zhengxilu, syl\}@zju.edu.cn\qquad guqi03@meituan.com} \\
  \begin{tabular}{@{}ll@{}}
  \end{tabular}}
\usepackage{graphicx}  
\usepackage{wrapfig}   
\usepackage{caption}   
\begin{document}

\ifcolmsubmission
\linenumbers
\fi

\maketitle
\vspace{-15pt}

\begin{abstract}
Reinforcement learning (RL) has emerged as a central paradigm for post-training LLM agents, yet its trajectory-level reward signal provides only coarse supervision for long-horizon interaction. 
On-Policy Self-Distillation (OPSD) complements RL by introducing dense token-level guidance from a teacher branch augmented with privileged context. 
However, transferring OPSD to multi-turn agents proves problematic: compounding multi-turn instability destabilizes supervision, while skill-conditioned privileged guidance requires asymmetric treatment for negative teacher rejections may arise from imperfect skills retrieval or utilization.
We introduce \textbf{SDAR} (\textbf{S}elf-\textbf{D}istilled \textbf{A}gentic \textbf{R}einforcement Learning), which treats OPSD as a gated auxiliary objective while keeping RL as the primary optimization backbone. 
\methodname{} maps detached token-level signals into a sigmoid gate, strengthening distillation on teacher-endorsed positive-gap tokens and softly attenuating negative teacher rejections. 
Across the Qwen2.5 and Qwen3 families on ALFWorld, WebShop, and Search-QA, \methodname{} substantially improves over GRPO (+9.4\% on ALFWorld, +7.0\% on Search-QA, +10.2\% on WebShop-Acc), avoids the instability of naive GRPO+OPSD, and consistently outperforms hybrid RL--OPSD baselines across model scales. Code available: \url{https://github.com/ZJU-REAL/SDAR}.


\end{abstract}

\begin{figure}[h]
\centering
\includegraphics[width=0.97\columnwidth]{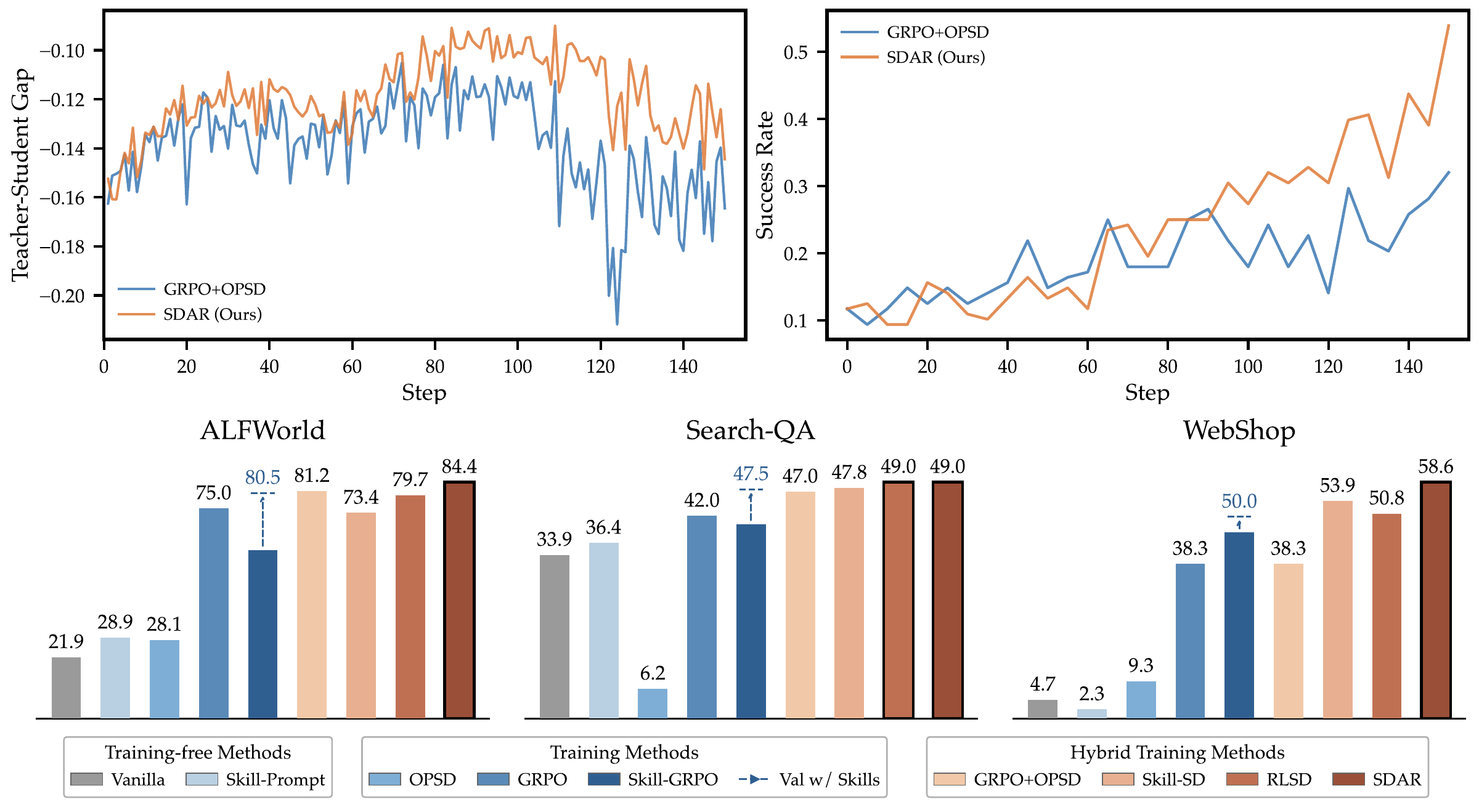}
\caption{(a) Comparison between GRPO+OPSD and \methodname{}; (b) Overall Performance.}
\label{fig:teaser}
\end{figure}
\section{Introduction}



Agentic post-training has become a central challenge for Large Language Models (LLMs)~\citep{guo2025ds-r1,team2025kimi,yang2025qwen3,comanici2025gemini,team2026longcat-2601}. 
Unlike static single-turn reasoning, multi-turn agents interact with environments over extended horizons, where each action changes future observations and each generated response becomes part of the context for subsequent decisions~\citep{shen2023hugginggpt,shi2025toollearning,jimenez2023swebench}.

Two paradigms naturally emerge as complementary forces:
Reinforcement Learning (RL)~\citep{shao2024deepseekmath,dong2025arpo,feng2025gigpo} provides task-level optimization grounded in environment or verifier feedback, whereas On-Policy Distillation (OPD)~\citep{ye2026opcd,yang2026g-opd,coreteam2026mimov2,glm5team2026glm5} and On-Policy Self-Distillation (OPSD)~\citep{zhao2026opsd,he2026sdzero,zhang2026embarrassinglysd} provide dense token-level guidance from a teacher branch. 
Yet, OPSD does not transfer cleanly to multi-turn agent training. 
We attribute this to two observations: \textbf{[1] Multi-turn OPSD Instability} and \textbf{[2] Asymmetric Trust in Privileged Guidance}.

\paragraph{[Observation-1] Multi-turn OPSD Instability} Once the student agent inevitably drifts

\vspace{-0.4\baselineskip}
\begin{wrapfigure}{r}{0.60\linewidth}
\vspace{-1\baselineskip}
\centering
\includegraphics[width=\linewidth]{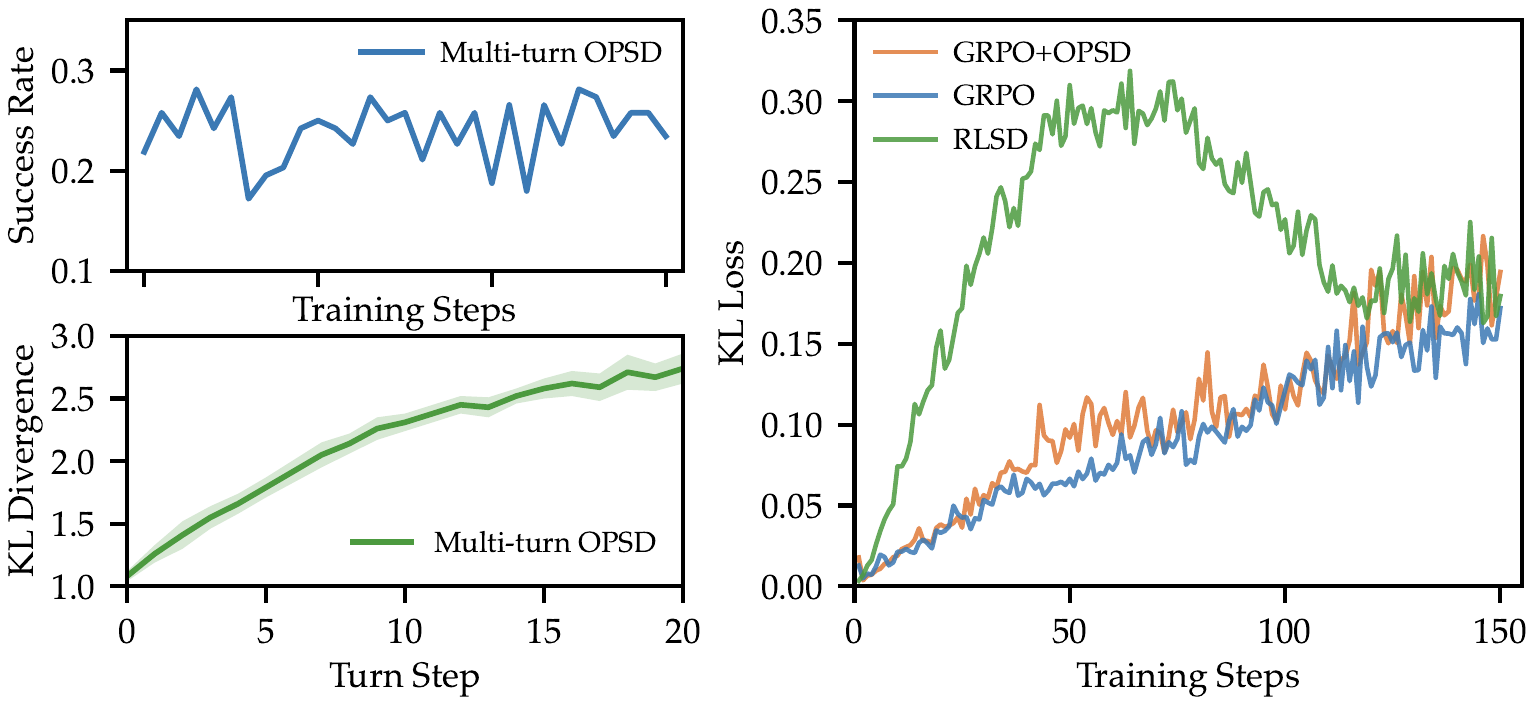}
\captionsetup{font=small, skip=6pt}
\caption{\textbf{Left}: Multi-turn OPSD Instability, with performance and KL reported. \textbf{Right}: RLSD-Style Instability, with KL loss.}
\label{fig:instability}
\vspace{-1.2\baselineskip}
\end{wrapfigure}
from the teacher-supported trajectory, the once-helpful token-level supervision becomes increasingly unreliable. This compounding error leads to surging per-turn KL divergence and catastrophic degradation in task performance, as shown in Figure~\ref{fig:instability} (Left).
TCOD~\citep{wang2026tcod} attempts to address this through curriculum learning, but relies on rigid temporal schedules or trajectory-depth thresholds.

 %
%


\paragraph{[Observation-2] Asymmetric Trust in Privileged Guidance.}
In OPSD, the teacher branch is not an independently stronger model, but the same policy augmented with privileged training-only context, such as retrieved skills. 
This makes its token-level guidance inherently asymmetric. 
For a student-sampled token $y_t$, if the privileged teacher assigns a higher probability than the student, the retrieved skill provides an endorsement signal: it supports an on-policy behavior that the student can already generate but has not fully internalized. 
Such positive guidance is particularly suitable for distillation.

In contrast, if the privileged teacher assigns a lower probability to the sampled token, the signal should be interpreted more cautiously. 
A negative gap may indicate that the token should indeed be suppressed, but in skill-conditioned OPSD it may also arise from the instability of privileged context: 
\textbf{(1)~Skill Quality.} Retrieved skills may be irrelevant, incomplete, or redundant. 
\textbf{(2)~Skill Utilization.} The teacher may fail to ground even relevant skills into reliable token-level preferences~\citep{chen2019learningcheating}. 
\textbf{(3)~Multi-turn Drift.} As trajectories unfold, the teacher-student gap tends to widen across turns (Figure~\ref{fig:gaps_analysis}, Middle), amplifying early mismatches over successive decisions~\citep{ross2011dagger}. 
Our preliminary study on Qwen2.5-3B-Instruct shows that negative-gap tokens exceed 50\% of all tokens (Figure~\ref{fig:gaps_analysis}), making this issue pervasive. 
This motivates an asymmetric treatment of privileged guidance: trust positive teacher endorsements more strongly, while applying negative teacher rejections more conservatively.

\begin{figure}[t]
\centering
\includegraphics[width=\columnwidth]{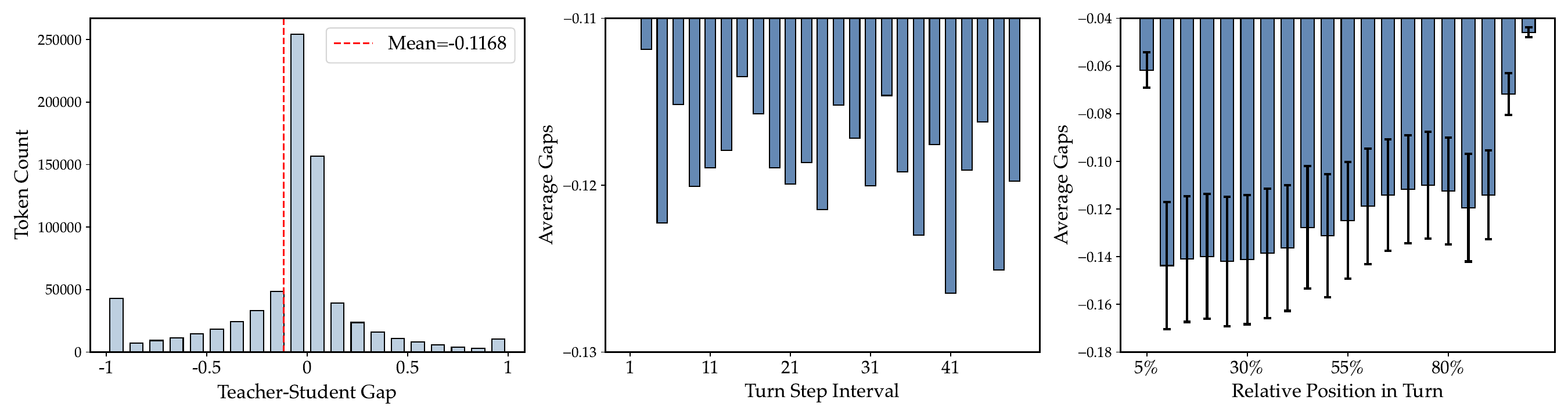}
\caption{\textbf{Teacher-Student Gap Analysis.} \textbf{Left}: Token count distribution partitioned by Teacher-Student gap value. \textbf{Middle}: Average teacher-student gap indexed by multi-turn step. \textbf{Right}: Average teacher-student gap indexed by relative position within a single turn.}
\label{fig:gaps_analysis}
\end{figure}

A stark realization emerges: for multi-turn agents, RL could reign as the primary optimization backbone, while OPSD is relegated to a carefully controlled auxiliary role.

But how should this auxiliary role be controlled? RLSD~\citep{yang2026rlsd} directly uses self-divergence to re-weight token-level RL advantages, but can substantially amplify updates 
especially early in training when teacher-student mismatch is large (see Figure~\ref{fig:instability}, Right).

We take a different path: the OPSD loss is treated as a direct, auxiliary optimization objective, leaving the verifier-driven RL policy loss untouched and thereby strictly preserving the semantics and unbiasedness of the RL advantage. To overcome instability of multi-turn OPSD and privileged guidance, distillation is not performed uniformly on every token. Instead, tokens are selectively distilled via an adaptive, smooth gating mechanism rather than a hand-crafted, rigid schedule (such as Skill-SD~\citep{wang2026skillsd} and HDPO~\citep{ding2026hdpo}). Inspired by TIP~\citep{xu2026tip}, we use token-level signals (such as student entropy or teacher-student divergence) to control the gate's activation. The core philosophy is simple: \emph{let each token decide the intensity of its own supervision.}
This yields a dynamic, self-paced curriculum operating at the finest possible granularity: the individual token level.

We validated our method across the Qwen2.5 and Qwen3 model families on three diverse benchmarks for llm-based agents: ALFWorld~\citep{shridhar2020alfworld}, WebShop~\citep{yao2022webshop}, and Search-QA~\citep{jin2025searchr1}. \methodname{} achieves substantial improvements over GRPO ($+9.4\%$ on ALFWorld, $+7.0\%$ on Search-QA, and $+10.2\%$ on WebShop-Acc for 7B), entirely avoids the catastrophic instability of na\"ive GRPO+OPSD, and consistently outperforms RL--OPSD hybrid methods such as Skill-SD and RLSD across all three model scales (Qwen3-1.7B included). Furthermore, robustness analysis shows that \methodname{} degrades gracefully with retrieval quality: even random retrieval outperforms the GRPO baseline, as our gating design filters out noise from low-quality skills and distills beneficial signals only. 


\section{Method}

\subsection{Problem Setup}

We consider a multi-turn agent that interacts with an environment over a finite horizon.
Given an initial prompt or task description $x$, at turn $k$ the agent receives an observation $o_k$,
generates a response $a_k$, and the environment returns the next observation $o_{k+1}$.
Each response $a_k$ may contain both intermediate reasoning tokens and executable action tokens.
For notational simplicity, we flatten all valid response tokens in one trajectory into a single token sequence
\[
y = (y_1,\dots,y_T) \sim \pi_{\theta}(\cdot \mid x),
\]
where $\pi_{\theta}$ denotes the student policy and $T$ is the total number of valid response tokens.
\begin{figure}[t]
\centering
\includegraphics[width=\columnwidth]{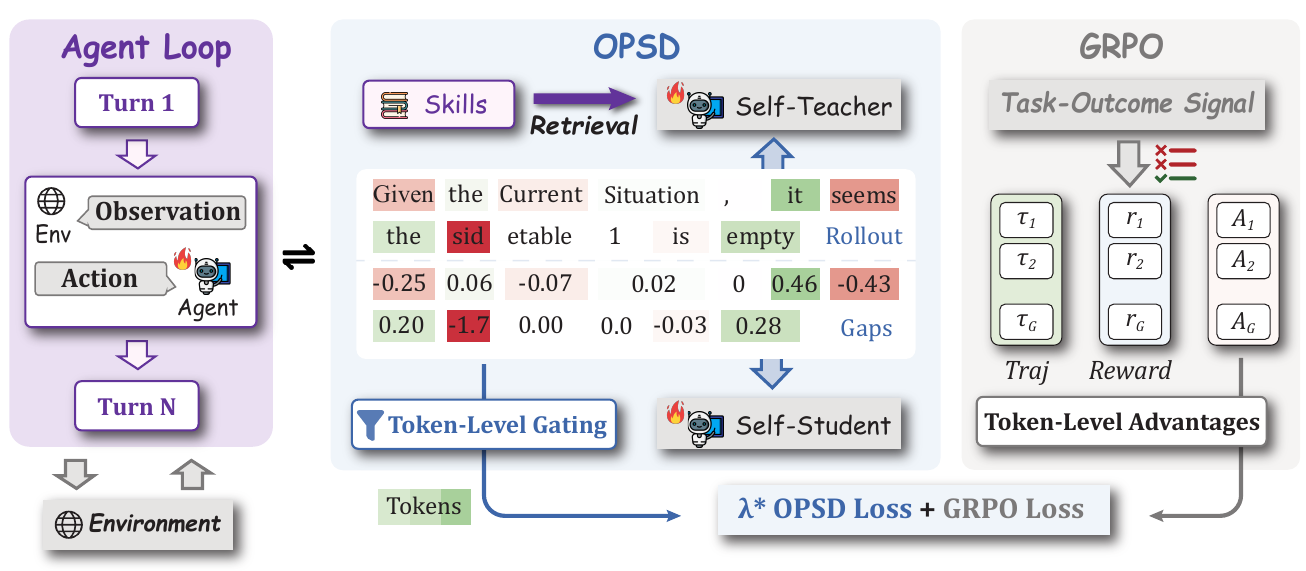}
\caption{\textbf{Illustrations of \methodname{} framework,} which trains multi-turn agents using token-level OPSD loss and verifier-driven RL loss.} 
\label{fig:method}
\end{figure}
At token position $t$, we denote the self-student context by
\[
s_t = (x, y_{<t}),
\]
and the self-teacher context by
\[
s_t^{+} = (x, c^{+}, y_{<t}),
\]
where $c^{+}$ denotes privileged training-only context available only to the teacher branch,
such as reference answers, skills (ours), or other auxiliary information not accessible at test time.
\paragraph{Skills Retrieval} We retrieve task-relevant \emph{skills}---compact, structured demonstrations that encode domain-specific knowledge such as sub-goal decompositions or action templates. 
We implement four retrieval strategies of varying quality to evaluate the robustness of our framework to the fidelity of the retrieved context: \textbf{(1) UCB Retrieval}, \textbf{(2) Keyword Matching (KM)}, \textbf{(3) Full Retrieval}, and \textbf{(4) Random Retrieval}.

Skill retrieval is cast as a multi-armed bandit problem over the skill library $\mathcal{E} = \{e_1, \dots, e_M\}$. For each incoming task, \textbf{UCB Retrieval} selects the single highest-scoring skill file according to the Upper Confidence Bound (UCB) criterion:
\begin{equation}
    \mathrm{score}(e) \;=\; \bar{r}(e) \;+\; c\,\sqrt{\frac{\ln N_{\mathrm{ucb}}}{n(e)}},
\end{equation}
where $\bar{r}(e)$ is the running mean reward obtained when skill $e$ was previously supplied as context, $N_{\mathrm{ucb}}$ is the total number of retrieval queries issued for the same task type, $n(e)$ is the number of times $e$ has been selected, and $c$ controls the exploration--exploitation trade-off.
\textbf{Keyword Matching} bypasses the bandit formulation and instead identifies the task
  scenario by matching keywords in the task description against predefined category labels,
  directly retrieving the skill file associated with the matched category.

\subsection{Optimization Goals}
\label{sec:opsd}
Our method is designed as an auxiliary objective on top of a standard policy optimization GRPO loss. The overall training objective is
\[
\mathcal{L}(\theta)
=
\mathcal{L}_{\text{GRPO}}(\theta)
+
\lambda_{\text{\methodname}} \cdot \mathcal{L}_{\text{\methodname}}(\theta),
\]
where $\mathcal{L}_{\text{GRPO}}$ is the original policy loss and
$\mathcal{L}_{\text{\methodname}}$ is our on-policy self-distillation objective.

Let $m_t \in \{0,1\}$ be the response mask indicating whether token $t$ is valid.
We define masked token averaging as
\[
\operatorname{Agg}(z_{1:T})
=
\frac{\sum_{t=1}^{T} m_t z_t}{\sum_{t=1}^{T} m_t}.
\]


\paragraph{RL Optimization} For each input $x$, GRPO samples a group of responses
\[
\{y^{(i)}\}_{i=1}^{G} \sim \pi_{\theta}(\cdot \mid x),
\]
and computes a sequence-level advantage $A^{(i)}$ from environment rewards.
Using a reference policy $\pi_{\mathrm{ref}}$, the GRPO objective can be written as
\begin{align}
\mathcal{L}_{\text{GRPO}}(\theta)
&=
-\frac{1}{G}\sum_{i=1}^{G}
\operatorname{Agg}\!\left(
\min\!\Big(
r_{t}^{(i)} A^{(i)},\,
\operatorname{clip}(r_{t}^{(i)},1-\epsilon,1+\epsilon)A^{(i)}
\Big)
\right)
\notag\\
&\quad
+\,\beta \cdot
\frac{1}{G}\sum_{i=1}^{G}
\operatorname{Agg}\!\left(
D_{\mathrm{KL}}\!\big(
\pi_{\theta}(\cdot \mid s_t^{(i)})
\,\|\,
\pi_{\mathrm{ref}}(\cdot \mid s_t^{(i)})
\big)
\right),
\end{align}
where $r_t^{(i)}=\pi_{\theta}(y_t^{(i)} \mid s_t^{(i)}) / \pi_{\theta_{\mathrm{old}}}(y_t^{(i)} \mid s_t^{(i)})$ is the importance sampling ratio.




\paragraph{OPSD Optimization}

At a fixed token position $t$, the teacher and student induce conditional token distributions $\pi_T(\cdot \mid s_t^{+})$ and $\pi_{\theta}(\cdot \mid s_t)$, respectively. The per-token reverse KL divergence is defined as:
$$
D_{\mathrm{RKL}}^{(t)} = D_{\mathrm{KL}}\!\left( \pi_{\theta}(\cdot \mid s_t) \;\middle\|\; \pi_T(\cdot \mid s_t^{+}) \right) = \sum_{v \in \mathcal{V}} \pi_{\theta}(v \mid s_t) \log \frac{\pi_{\theta}(v \mid s_t)}{\pi_T(v \mid s_t^{+})}.
$$
To efficiently derive an \emph{importance signal} without computing the expensive full-vocabulary summation, we take a single-sample estimate on the student-sampled token $y_t \sim \pi_{\theta}(\cdot \mid s_t)$. The negation of this estimate directly yields the Teacher-Student log-probability gap $\Delta_t$:
$$
\Delta_t = -\widehat{D}_{\mathrm{RKL}}^{(t)} = \log \pi_T(y_t \mid s_t^{+}) - \log \pi_{\theta}(y_t \mid s_t).
$$

\subsection{Token-Level Gating}
\label{sec:token_gating}

The key idea is to convert privileged teacher guidance into a token-level trust weight, 
while keeping the verifier-driven RL objective unchanged. 
We introduce a token-level gate $g_t\in[0,1]$ that modulates the OPSD signal on each student-sampled token,
and apply it to a sampled-token surrogate so that different gating strategies share the same optimization.

Let
\[
\Delta_t
\;=\;
\operatorname{sg}\!\Bigl(
  \log\pi_{\theta}^{+}(y_t\mid s_t^{+})
  \;-\;
  \log\pi_{\theta}(y_t\mid s_t)
\Bigr)
\]
denote the \emph{detached} Teacher-Student log-probability gap on the student-sampled token, and
\[
h_t = -\sum_{v\in\mathcal{V}} \pi_{\theta}(v\mid s_t)\,\log\pi_{\theta}(v\mid s_t)
\]
denote the student entropy at position~$t$.
We compose each raw score with the logistic sigmoid~$\sigma$
so that every gate is smooth, differentiable, and naturally bounded in~$(0,1)$.
The sharpness parameter $\beta>0$ controls the transition between conservative attenuation and strong activation.

We instantiate three complementary gating strategies:
\begin{enumerate}
  \item \emph{Entropy gating}: $g_t = \sigma(\beta\,h_t)$,
    targeting high-entropy positions where the student is most uncertain.
  \item \emph{Gap gating}: $g_t = \sigma(\beta\,\Delta_t)$,
    assigning larger weights to positive-gap tokens endorsed by the privileged teacher while attenuating negative-gap tokens.
  \item \emph{Soft-OR gating}: $g_t = \sigma\!\bigl(\beta\bigl[1 - (1-h_t)(1-\Delta_t)\bigr]\bigr)$,
    combining student uncertainty and teacher-student gap as an alternative gating strategy.
\end{enumerate}

In all cases, the gate is detached via $\operatorname{sg}(\cdot)$,
so gradients flow exclusively through the student log-probability.
The token-level loss is
\[
\ell_t^{\,\methodname}
\;=\;
g_t \cdot
\bigl(
  \log\pi_{\theta}^{+}(y_t\mid s_t^{+})
  -
  \log\pi_{\theta}(y_t\mid s_t)
\bigr),
\qquad
\mathcal{L}_{\methodname}
\;=\;
\operatorname{Agg}\!\bigl(\ell_t^{\,\methodname}\bigr).
\]
With gap gating, the sigmoid gate implements asymmetric token-level modulation: positive-gap tokens receive stronger auxiliary distillation, while negative-gap tokens are softly attenuated.
We also provide theoretical analysis of our design in Appendix~\ref{appendix:proof}.

\begin{table*}[t]
    \centering
    \caption{
        \textbf{Performance on ALFWorld, Search-QA and WebShop tasks.}
        We report the success rate (\%) on ALFWorld, accuracy (\%) on Search-QA, and Score/Acc (\%) on WebShop (128 tasks). * means validation with skills.
        \sethlcolor{topcolor}\hl{\textbf{Best}} and \sethlcolor{secondcolor}\hl{\mbox{\underline{second-best}}} are highlighted. 
    }
    \label{tab:main_results}
    \resizebox{1\textwidth}{!}{%
    \begin{tabular}{l CCCCCCC CCCCCCCC CC }
    \toprule
    & \multicolumn{7}{c}{\textbf{ALFWorld}} & \multicolumn{8}{c}{\textbf{Search-QA}} & \multicolumn{2}{c}{\textbf{WebShop}} \\
    \cmidrule(lr){2-8} \cmidrule(lr){9-16} \cmidrule(lr){17-18}
    \textbf{Method}
    & \textbf{Pick} & \textbf{Look} & \textbf{Clean} & \textbf{Heat} & \textbf{Cool} & \textbf{Pick2} & \textbf{Avg}
    & \textbf{NQ} & \textbf{Triv} & \textbf{Pop} & \textbf{Hotp} & \textbf{2Wk} & \textbf{MuS} & \textbf{Bam} & \textbf{Avg}
    & \textbf{Score} & \textbf{Acc} \\
    \midrule

    \rowcolor{gray!10} \multicolumn{18}{l}{\textit{Qwen2.5-3B-Instruct}} \\

    Vanilla
        & 44.4 & 11.1 & 6.2 & 15.4 & 28.6 & 12.5 & 21.9
        & 24.6 & 48.1 & 31.0 & 26.3 & 25.3 & 7.2 & 59.7 & 31.7
        & 6.7 & 0.8
        \\
    Skill-Prompt*
        & 51.7 & 66.7 & 48.4 & 0.0 & 4.3 & 10.0 & 28.9
        & 23.7 & 46.2 & 30.6 & 24.4 & 22.1 & 7.5 & 12.5 & 23.9
        & 0.2 & 0.8
        \\
    OPSD
        & 48.8 & 41.7 & 16.7 & 0.0 & 15.8 & 16.7 & 28.1
        & 0.1 & 0.1 & 0.1 & 0.0 & 0.0 & 0.0 & 0.0 & 0.0
        & 11.3 & 3.1
        \\

    GRPO
        & 91.2 & 62.5 & \cellcolor{secondcolor}\underline{96.2} & 61.9 & 65.0 & 47.4 & 75.0
        & 39.3 & \cellcolor{secondcolor}\underline{60.6} & 41.1 & 37.4 & 34.6 & 15.4 & 26.4 & 36.4
        & 79.8 & 63.3
        \\
    Skill-GRPO
        & 88.9 & 71.4 & 58.8 & \cellcolor{secondcolor}\underline{70.6} & 40.7 & 29.2 & 60.2
        & 43.5 & 58.8 & 43.0 & 36.8 & 32.2 & 11.7 & 12.5 & 34.1
        & 77.3 & 60.9
        \\
    Skill-GRPO*
        & 94.3 & 57.1 & \cellcolor{topcolor}\textbf{100} & 66.7 & \cellcolor{secondcolor}\underline{73.1} & 57.1 & 80.5
        & 44.3 & 59.6 & \cellcolor{secondcolor}\underline{44.3} & 39.0 & 36.1 & 14.5 & 14.9 & 36.1
        & 76.3 & \cellcolor{secondcolor}\underline{66.4}
        \\
    GRPO+OPSD
        & \cellcolor{topcolor}\textbf{100} & \cellcolor{topcolor}\textbf{82.4} & 85.7 & \cellcolor{topcolor}\textbf{75.0} & 70.0 & 60.0 & \cellcolor{secondcolor}\underline{81.2}
        & \cellcolor{topcolor}\textbf{44.9} & \cellcolor{topcolor}\textbf{61.2} & \cellcolor{topcolor}\textbf{45.2} & \cellcolor{topcolor}\textbf{40.4} & 38.5 & \cellcolor{secondcolor}\underline{16.0} & \cellcolor{secondcolor}\underline{66.1} & \cellcolor{topcolor}\textbf{44.6}
        & 77.8 & \cellcolor{secondcolor}\underline{66.4}
        \\
    Skill-SD
        & 88.2 & 50.0 & \cellcolor{secondcolor}\underline{96.2} & 52.4 & 65.0 & 57.9 & 73.4
        & 44.4 & 60.4 & 44.0 & \cellcolor{secondcolor}\underline{39.5} & \cellcolor{topcolor}\textbf{40.4} & 15.4 & 64.9 & \cellcolor{secondcolor}\underline{44.1}
        & 75.9 & 64.0
        \\
    RLSD
        & 87.9 & \cellcolor{secondcolor}\underline{75.0} & 90.9 & \cellcolor{topcolor}\textbf{75.0} & \cellcolor{secondcolor}\underline{73.1} & \cellcolor{secondcolor}\underline{68.4} & 79.7
        & 41.5 & 58.6 & 42.3 & \cellcolor{topcolor}\textbf{40.4} & \cellcolor{secondcolor}\underline{40.2} & \cellcolor{topcolor}\textbf{16.8} & \cellcolor{topcolor}\textbf{66.9} & 43.8
        & \cellcolor{secondcolor}\underline{84.4} & \cellcolor{secondcolor}\underline{66.4}
        \\
    \textbf{\methodname{}}
        & \cellcolor{secondcolor}\underline{97.1} & 62.5 & \cellcolor{topcolor}\textbf{100} & 61.9 & \cellcolor{topcolor}\textbf{75.0} & \cellcolor{topcolor}\textbf{84.2} & \cellcolor{topcolor}\textbf{84.4}
        & \cellcolor{secondcolor}\underline{44.8} & 58.1 & \cellcolor{secondcolor}\underline{44.3} & 38.6 & 36.2 & 15.7 & \cellcolor{secondcolor}\underline{66.1} & 43.4
        & \cellcolor{topcolor}\textbf{85.0} & \cellcolor{topcolor}\textbf{68.0}
        \\

    \midrule

    \rowcolor{gray!10} \multicolumn{18}{l}{\textit{Qwen2.5-7B-Instruct}} \\

    Vanilla
        & 36.1 & 22.2 & 3.1 & 0.0 & 0.0 & 0.0 & 12.5
        & 25.2 & 50.8 & 29.5 & 29.0 & 29.0 & 10.4 & 63.7 & 33.9
        & 5.9 & 1.6
        \\
    Skill-Prompt*
        & 51.7 & 50.0 & 32.3 & 5.3 & 4.3 & 0.0 & 23.4
        & 30.9 & 52.1 & 32.7 & 32.7 & 27.9 & 12.7 & 66.1 & 36.4
        & 1.7 & 0.8
        \\
    OPSD
        & 50.0 & 60.0 & 22.7 & 21.4 & 17.6 & 9.5 & 32.8
        & 8.8 & 8.6 & 17.5 & 2.5 & 4.2 & 0.5 & 1.2 & 6.2
        & 4.5 & 2.3
        \\

    GRPO
        & 91.2 & \cellcolor{secondcolor}\underline{87.5} & 96.2 & 81.0 & 65.0 & 57.9 & 81.2
        & 45.1 & \cellcolor{secondcolor}\underline{63.7} & 44.0 & 43.6 & 43.2 & 16.8 & 37.6 & 42.0
        & 80.9 & 72.6
        \\
    Skill-GRPO
        & 88.5 & 66.7 & 65.2 & 61.1 & 57.7 & \cellcolor{secondcolor}\underline{73.1} & 69.5
        & 45.2 & \cellcolor{secondcolor}\underline{63.7} & 45.7 & 43.1 & 43.3 & 19.6 & 21.4 & 40.3
        & 80.4 & 71.9
        \\
    Skill-GRPO*
        & \cellcolor{topcolor}\textbf{100} & 83.3 & \cellcolor{secondcolor}\underline{96.4} & 83.3 & 75.0 & \cellcolor{topcolor}\textbf{78.9} & \cellcolor{topcolor}\textbf{88.3}
        & 44.8 & 63.0 & 45.1 & 43.7 & 43.7 & \cellcolor{secondcolor}\underline{20.5} & 71.4 & 47.5
        & 87.0 & \cellcolor{secondcolor}\underline{81.2}
        \\
    GRPO+OPSD
        & 91.4 & 61.5 & \cellcolor{topcolor}\textbf{100} & \cellcolor{secondcolor}\underline{87.5} & \cellcolor{secondcolor}\underline{76.5} & 52.2 & 80.4
        & \cellcolor{topcolor}\textbf{47.3} & \cellcolor{topcolor}\textbf{64.5} & 46.9 & 43.8 & 39.3 & 18.0 & 69.4 & 47.0
        & 86.8 & 76.5
        \\
    Skill-SD
        & 93.9 & \cellcolor{topcolor}\textbf{93.8} & 90.9 & \cellcolor{topcolor}\textbf{100} & 69.2 & 68.4 & 85.1
        & \cellcolor{secondcolor}\underline{47.1} & \cellcolor{topcolor}\textbf{64.5} & \cellcolor{secondcolor}\underline{47.8} & \cellcolor{secondcolor}\underline{44.2} & 42.1 & 20.2 & 69.0 & 47.8
        & 86.1 & 76.5
        \\
    RLSD
        & \cellcolor{topcolor}\textbf{100} & \cellcolor{secondcolor}\underline{87.5} & 92.3 & 58.8 & \cellcolor{topcolor}\textbf{80.0} & 65.2 & 82.0
        & 46.8 & 63.0 & 44.4 & \cellcolor{topcolor}\textbf{45.5} & \cellcolor{topcolor}\textbf{48.9} & \cellcolor{topcolor}\textbf{21.5} & \cellcolor{secondcolor}\underline{73.0} & \cellcolor{secondcolor}\underline{49.0}
        & \cellcolor{secondcolor}\underline{87.4} & 77.3
        \\
    \textbf{\methodname{}}
        & \cellcolor{secondcolor}\underline{94.7} & 75.0 & \cellcolor{topcolor}\textbf{100} & 86.7 & 68.2 & \cellcolor{topcolor}\textbf{78.9} & \cellcolor{secondcolor}\underline{85.9}
        & 46.3 & 63.5 & \cellcolor{topcolor}\textbf{48.2} & 43.8 & \cellcolor{secondcolor}\underline{48.4} & 19.6 & \cellcolor{topcolor}\textbf{73.0} & \cellcolor{topcolor}\textbf{49.0}
        & \cellcolor{topcolor}\textbf{89.4} & \cellcolor{topcolor}\textbf{82.8}
        \\

    \midrule

    \rowcolor{gray!10} \multicolumn{18}{l}{\textit{Qwen3-1.7B-Instruct}} \\

    Vanilla
        & 25.0 & 22.2 & 3.1 & 0.0 & 21.4 & 4.2 & 12.5
        & 29.4 & 46.9 & 37.0 & 23.5 & 19.6 & 6.4 & 10.5 & 24.8
        & 46.5 & 4.7
        \\
    Skill-Prompt*
        & 10.3 & \cellcolor{secondcolor}\underline{50.0} & 16.1 & 0.0 & 0.0 & 5.0 & 9.4
        & 29.4 & 46.5 & 36.2 & 22.9 & 20.8 & 4.3 & 10.1 & 24.3
        & 23.0 & 2.3
        \\
    OPSD
        & 26.3 & 33.3 & 9.1 & 0.0 & 4.5 & 5.3 & 14.1
        & 4.2 & 8.3 & 4.6 & 6.6 & 15.3 & 0.7 & 1.2 & 5.8
        & 47.4 & 9.3
        \\

    GRPO
        & \cellcolor{secondcolor}\underline{71.1} & 41.7 & 36.4 & \cellcolor{secondcolor}\underline{40.0} & 31.8 & \cellcolor{secondcolor}\underline{31.6} & 46.1
        & \cellcolor{secondcolor}\underline{40.0} & \cellcolor{topcolor}\textbf{58.9} & 43.5 & 35.4 & 30.3 & 12.0 & 65.7 & 40.8
        & 67.3 & 38.3
        \\
    Skill-GRPO
        & 27.6 & \cellcolor{topcolor}\textbf{54.5} & 22.7 & 27.3 & 0.0 & 19.2 & 21.1
        & 39.2 & \cellcolor{secondcolor}\underline{58.6} & 43.9 & 35.2 & 28.2 & 11.5 & \cellcolor{secondcolor}\underline{66.1} & 40.4
        & 73.4 & 46.1
        \\
    Skill-GRPO*
        & 31.4 & 42.9 & 51.9 & 8.3 & 11.5 & 7.1 & 28.1
        & 38.0 & 58.4 & 43.9 & \cellcolor{secondcolor}\underline{36.3} & 29.0 & 12.5 & \cellcolor{topcolor}\textbf{66.9} & 40.7
        & \cellcolor{secondcolor}\underline{80.4} & 50.0
        \\
    GRPO+OPSD
        & 38.2 & \cellcolor{secondcolor}\underline{50.0} & 30.8 & 28.6 & 30.0 & 21.1 & 32.0
        & \cellcolor{topcolor}\textbf{40.7} & \cellcolor{topcolor}\textbf{58.9} & 45.0 & \cellcolor{topcolor}\textbf{37.0} & \cellcolor{secondcolor}\underline{34.6} & \cellcolor{topcolor}\textbf{13.3} & 65.7 & \cellcolor{topcolor}\textbf{42.2}
        & 70.7 & 38.3
        \\
    Skill-SD
        & 52.9 & 37.5 & \cellcolor{secondcolor}\underline{69.2} & \cellcolor{topcolor}\textbf{42.9} & \cellcolor{topcolor}\textbf{60.0} & \cellcolor{topcolor}\textbf{36.8} & \cellcolor{secondcolor}\underline{52.3}
        & 39.1 & 57.5 & \cellcolor{topcolor}\textbf{45.4} & 34.8 & 34.1 & 10.7 & 64.1 & 40.8
        & \cellcolor{topcolor}\textbf{81.8} & \cellcolor{secondcolor}\underline{53.9}
        \\
    RLSD
        & 50.0 & 37.5 & 61.5 & 19.0 & \cellcolor{secondcolor}\underline{50.0} & 21.1 & 42.2
        & 38.6 & 57.3 & 43.0 & 34.5 & 34.1 & 11.5 & 65.3 & 40.6
        & 74.0 & 50.8
        \\
    \textbf{\methodname{}}
        & \cellcolor{topcolor}\textbf{73.5} & 25.0 & \cellcolor{topcolor}\textbf{76.9} & 33.3 & 40.0 & \cellcolor{topcolor}\textbf{36.8} & \cellcolor{topcolor}\textbf{53.9}
        & 39.7 & \cellcolor{topcolor}\textbf{58.9} & \cellcolor{secondcolor}\underline{45.3} & 35.9 & \cellcolor{topcolor}\textbf{35.5} & \cellcolor{secondcolor}\underline{12.6} & 65.3 & \cellcolor{secondcolor}\underline{41.9}
        & 76.8 & \cellcolor{topcolor}\textbf{58.6}
        \\

    \bottomrule
    \end{tabular}
    }
    \end{table*}

\section{Experiment}
\paragraph{Benchmarks}

We evaluate our methods on ALFWorld~\citep{shridhar2020alfworld}, Search-based QA~\citep{jin2025searchr1}, and Webshop~\citep{yao2022webshop}. \textit{ALFWorld} is a text-based game aligned with the ALFRED embodied AI benchmark, including 3,827 task instances across six categories of common household activities: Pick and Place (Pick), Look at Obj in Light (Look), Pick Clean then Place in Recep (Clean), Pick Heat then Place in Recep (Heat), Pick Cool then Place in Recep (Cool), and Pick Two Obj and Place (Pick2). \textit{Search-based QA} contains several widely-used search-augmented QA benchmarks, including single-hop QA datasets (NQ~\citep{kwiatkowski2019nq}, TriviaQA~\citep{joshi2017triviaqa}, and PopQA~\citep{mallen2023popqa}) and multi-hop QA datasets (HotpotQA~\citep{yang2018hotpotqa}, 2Wiki~\citep{ho20202wiki}, MuSiQue~\citep{trivedi2022musique}, and Bamboogle~\citep{press2023bamboogle}).
\textit{WebShop} is a complex, web-based interactive environment designed to test the LLM agents in realistic online shopping scenarios. Agents navigate a realistic web interface to find and purchase products matching user specifications. We select 128 fixed tasks in validation set, which aligns with \citet{feng2025gigpo}.

\paragraph{Implementation Details.}
We train the Qwen2.5-Instruct and Qwen3-Instruct series using \methodname{} for at 150 steps on 8 H800 GPUs.
For ALFWorld, we adopt the training data split from GiGPO~\citep{feng2025gigpo}, 
with each batch sampling 16 tasks and 8 rollouts per prompt, 
and a maximum prompt length of 2,048 tokens.
For Search-QA, we follow the experimental setup of Search-R1~\citep{jin2025searchr1}, 
using E5~\citep{wang2022e5} as the retriever.
The training data are drawn from NQ and HotpotQA, making these two benchmarks in-domain, 
while the remaining datasets serve as out-of-domain evaluation.
Each batch samples 128 tasks with a maximum prompt length of 4,096 tokens. For Webshop, 1000 tasks are selected for training, with each batch sampling 16 tasks and 8 rollouts per prompt, and a maximum prompt length of 4,096 tokens. 
We set the \texttt{SkillBank} from SkillRL~\citep{xia2026skillrl} for all three environments. We set $\lambda_{\methodname}=0.01$ and $\beta=5.0$ in our experiments.

\paragraph{Baselines}
We compare \methodname{} against three categories of methods on three base models.
  \textbf{(1) Training-free methods.}
  \textit{Skill-Prompt} retrieves task-relevant skills from the \texttt{SkillBank} via keyword
  matching (KM) and prepends them to the input prompt at inference time.
  \textbf{(2) Post-training methods,} such as GRPO~\citep{shao2024deepseekmath}, OPSD~\citep{zhao2026opsd} and Skill-GRPO. \textit{Skill-GRPO} augments GRPO by retrieving skills via KM and injecting them into the
  training prompt;
  at test time it can run with (\textit{Skill-GRPO*}) or without retrieved skills.
  \textbf{(3) Hybrid methods}, that combine RL with privileged knowledge distillation, such as GRPO+OPSD, and Skill-SD~\citep{wang2026skillsd}, RLSD~\citep{yang2026rlsd}.  
  \textit{GRPO+OPSD} simply adds the OPSD distillation loss as an auxiliary objective on top of GRPO training. All the algorithms of \methodname{} and other baselines are detailed in Appendix~\ref{appendix:proof}.
\subsection{Main Results}
\paragraph{Overall Performance.}          
As summarized in Table~\ref{tab:main_results}, \methodname{} demonstrates exceptional performance, achieving the best or second-best results across almost all settings. Compared to GRPO, it delivers substantial gains: on Qwen2.5-3B, it improves ALFWorld by +9.4\% (84.4 vs.\ 75.0), Search-QA by +7.0\%, and WebShop-Acc by +4.7\%, with similarly consistent improvements on the 7B model. While standalone OPSD collapses catastrophically (near-zero on Search-QA) and a naive GRPO+OPSD combination degrades severely on Qwen3-1.7B (32.0 vs.\ 46.1) due to unbounded distillation gradients overwhelming the RL signal, \methodname{} avoids the observed instability and maintains stable gains. Through its adaptive gating mechanism, it ensures stable optimization and consistent gains across all model scales.
 
\paragraph{Skills Internalization.}                                                           
Beyond overall performance, \methodname{} successfully \emph{internalizes} privileged knowledge rather than superficially relying on it at inference~\citep{lu2026skill0}. While Skill-GRPO shows a massive performance drop when tested without skills (e.g., 60.2 vs.\ 80.5 on ALFWorld-3B) and even underperforms vanilla GRPO due to harmful distributional dependencies, \methodname{} requires no external skills during inference. Yet, it surpasses even the skill-augmented Skill-GRPO* in most settings, achieving 84.4 on ALFWorld-3B and a striking 53.9 (vs.\ 28.1) on ALFWorld-1.7B. These consistent gains confirm that our token-level gated distillation genuinely transfers underlying knowledge into the policy's parameters.

\paragraph{Strong Generalization.}                                                            
\methodname{} also exhibits stronger generalization compared to hybrid baselines such as Skill-SD and RLSD.
On Qwen2.5-3B, it outperforms both methods on ALFWorld (84.4 vs.\ 73.4 for Skill-SD and 79.7 for RLSD) and WebShop.
This advantage is most pronounced on the challenging Qwen3-1.7B model, where smaller models may struggle to utilize retrieved skills effectively.
In this regime, Skill-GRPO drops to 21.1\% on ALFWorld, well below GRPO's 46.1\%, and RLSD reaches 42.2\%.
In contrast, \methodname{} achieves the highest score of 53.9\%.
By attenuating uncertain negative teacher guidance while preserving positive teacher endorsements, our gating mechanism provides a more robust way to incorporate privileged knowledge without sacrificing generalization.

\subsection{Training Dynamics}                                                          
To elucidate the adaptive behavior of \methodname{} throughout RL optimization, we monitor two key metrics for the Qwen2.5-7B backbone on ALFWorld in Figure~\ref{fig:7b_alfworld_gap_gate}. 
\textit{(a)} shows that the mean Teacher-Student log-probability gap 
($\bar\Delta = \mathbb{E}_t[\Delta_t]$) remains consistently negative, 
indicating that the privileged teacher assigns lower probability than the student to sampled tokens on average.
This reveals partial asymmetric trust in privileged guidance regime where na\"ive distillation would actively degrade performance. Crucially, $\bar\Delta$ steadily converges toward zero, confirming that the gating mechanism successfully identifies and up-weights the specific subset of tokens where the teacher \emph{does} provide beneficial signals. To further validate this adaptive filtering, \textit{(b)} tracks the gate activation ratio (the fraction of tokens where $g_t > 0.5$). For the majority of early training, this ratio remains strictly below $0.5$, correctly suppressing tokens that carry negative signals. However, as the student's policy evolves, the ratio gradually increases, reflecting that more tokens enter a regime of constructive teacher guidance.

\begin{figure}[t]
\centering
\includegraphics[width=\columnwidth]{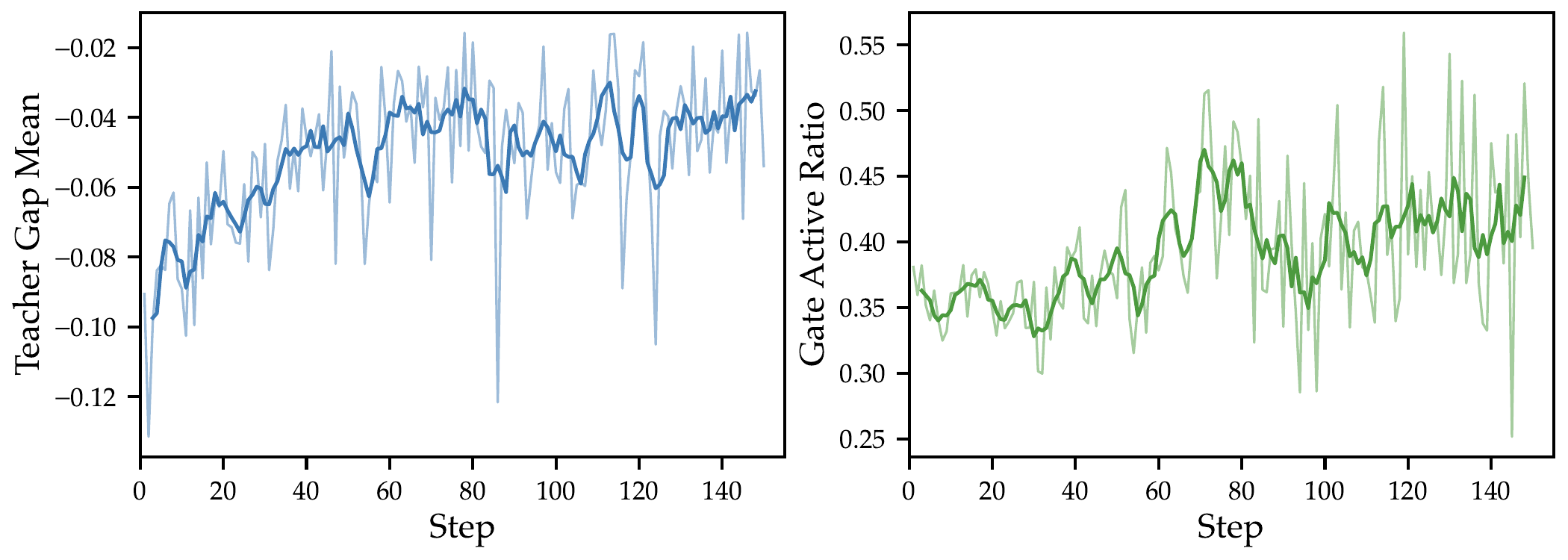}
\caption{\textbf{Training Dynamics.} Average teacher-student gap (Left) and gate activation ratio (Right) during the training of Qwen2.5-7B-Instruct on ALFWorld.} 
\label{fig:7b_alfworld_gap_gate}
\end{figure}

\subsection{Robust Analysis}


To address the practical concern of whether \methodname{} heavily relies on high-quality skill retrieval, we fix our optimal configuration ($\lambda=0.01$, $\beta=5.0$) and evaluate performance across four retrieval quality tiers (Table~\ref{tab:ablation_method}). 
All four strategies consistently outperform the pure GRPO baseline (\emph{w/o OPSD}). 
Even \textbf{Random Retrieval}---which selects skills with zero task awareness---yields gains of $+1.9$/$+1.6$/$+1.0$ on ALFWorld/WebShop-Score/WebShop-Acc. 
Higher-quality retrieval further amplifies these benefits: \textbf{Keyword Matching} achieves gains of $+4.7$/$+8.5$/$+10.2$ and even surpasses UCB on WebShop. 

These results echo our observation on asymmetric privileged guidance. 
Low-quality retrieval can introduce mismatched or unstable teacher signals, especially negative guidance from irrelevant skills. 
Rather than uniformly following such signals, \methodname{} uses token-level gating to retain positive teacher endorsements while softly attenuating uncertain negative rejections. 
Thus, the performance gains remain robust across retrieval qualities, suggesting that the uplift stems primarily from gated distillation rather than retrieval fidelity alone.

\begin{table}[t]       
    \centering
    \caption{Robust Testing of different skill retrieval methods.}                          
    \label{tab:ablation_method}                                                             
    \begin{tabular}{lccc}
    \toprule
        Method & ALFWorld & WebShop-Score & WebShop-Acc \\

        \midrule

        UCB & 86.8\posval{$_{+5.6}$} & 87.5\posval{$_{+6.6}$} & 81.2\posval{$_{+8.6}$} \\   
        KM & 85.9\posval{$_{+4.7}$} & 89.4\posval{$_{+8.5}$} & 82.8\posval{$_{+10.2}$} \\
        Full & 83.2\posval{$_{+2.0}$} & 87.2\posval{$_{+6.3}$} & 78.1\posval{$_{+5.5}$} \\  
        Random & 83.1\posval{$_{+1.9}$} & 82.5\posval{$_{+1.6}$} & 73.6\posval{$_{+1.0}$} \\
        \rowcolor{gray!10} w/o OPSD & 81.2 & 80.9 & 72.6 \\
        \bottomrule
    \end{tabular}
\end{table} 
\subsection{Ablation Studies}

\paragraph{Token-Level Gating Strategy.} As shown in Figure~\ref{fig:ablation_tip}, Teacher-Student Gap gating consistently outperforms both the entropy and soft-OR gating strategies (introduced in Section~\ref{sec:token_gating}), achieving a higher asymptotic success rate (${\sim}0.84$) and a steeper performance climb after the initial 100 steps. We attribute this superiority to the directness of the Teacher-Student gap ($\Delta_t$) as an importance signal, which precisely measures the teacher's disagreement with the student's chosen token. In contrast, entropy ($h_t$) acts as an indirect proxy that may erroneously activate on uncertain but already well-handled tokens, while soft-OR dilutes the gating signal by triggering when only one score is moderately large, thereby reducing its selectivity. All remaining experiments default to gap gating.
\begin{figure}[t]
\centering
\begin{minipage}{0.48\columnwidth}
    \centering
    \includegraphics[width=\linewidth]{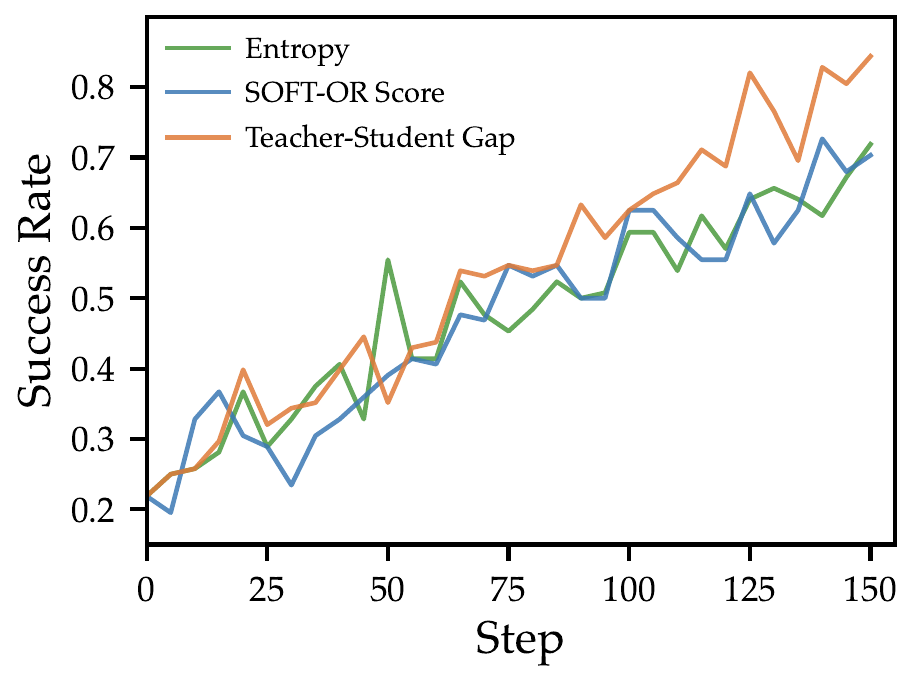}
    \caption{Ablations of Token-level Gating on Qwen2.5-3B-Instruct.}
    \label{fig:ablation_tip}
\end{minipage}
\hfill
\begin{minipage}{0.48\columnwidth}
    \centering
    \includegraphics[width=\linewidth]{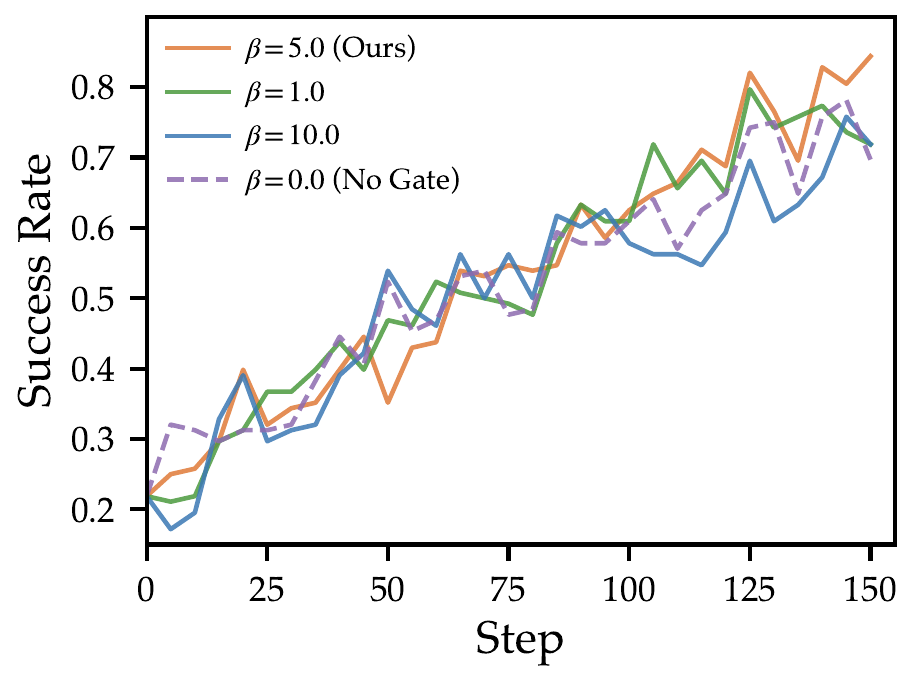}
    \caption{Ablations of $\beta$ on Qwen2.5-3B-Instruct.}
    \label{fig:ablation_beta}
\end{minipage}
\end{figure}
\begin{figure}[t]
\centering
\begin{minipage}{0.48\columnwidth}
    \centering
    \includegraphics[width=\linewidth]{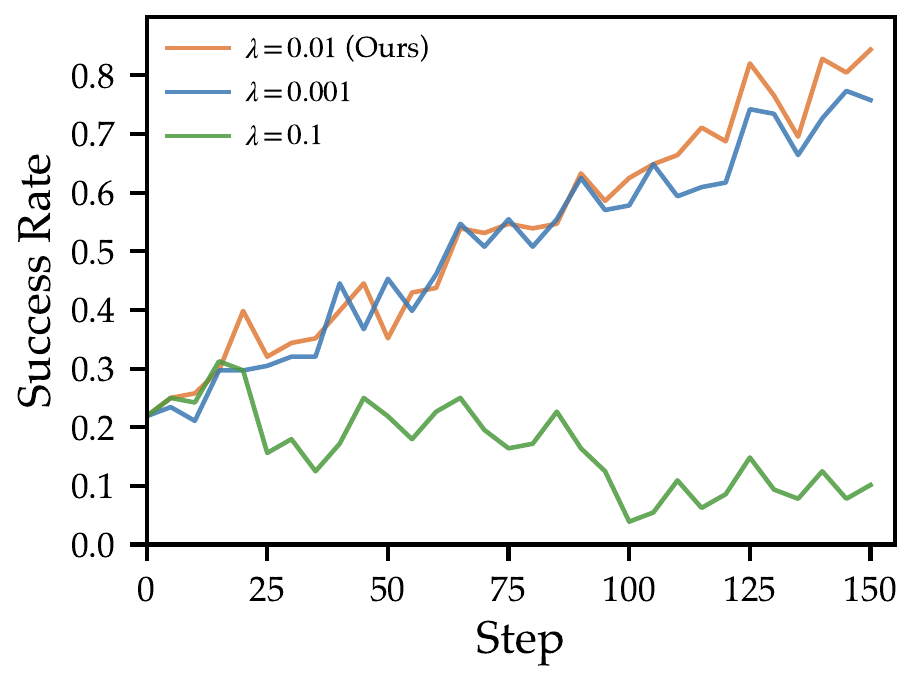}
    \caption{Ablations of $\lambda$ on Qwen2.5-3B-Instruct.}
    \label{fig:ablation_lambda}
\end{minipage}
\hfill
\begin{minipage}{0.48\columnwidth}
    \centering
    \includegraphics[width=\linewidth]{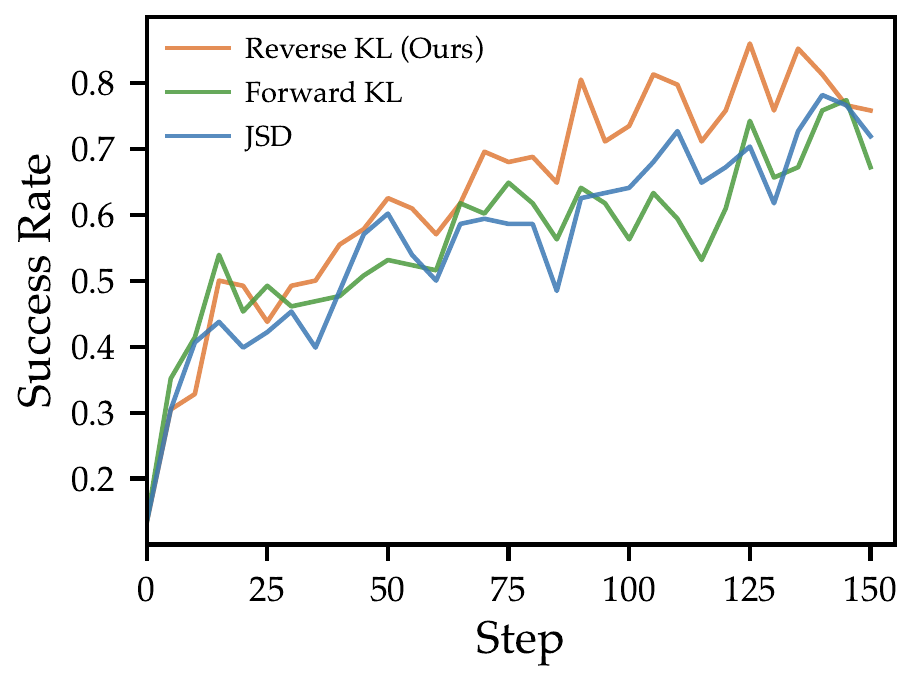}
    \caption{Ablations of $\mathcal{L}_{\methodname}$ type on Qwen2.5-7B-Instruct.}
    \label{fig:ablation_loss}
\end{minipage}
\end{figure}
\paragraph{Sharpness $\beta$.} Figure~\ref{fig:ablation_beta} evaluates the impact of sigmoid sharpness across $\beta \in \{0, 1, 5, 10\}$, where $\beta = 0$ denotes the complete removal of the gating mechanism (i.e., uniform distillation). The optimal performance is achieved at $\beta = 5$, which effectively balances two distinct failure modes: an excessively small $\beta$ (including the no-gate baseline) applies distillation indiscriminately, thereby inheriting the multi-turn instability of na\"ive OPSD; conversely, an overly large $\beta$ strictly binarizes the gate, stripping away the smooth modulation necessary for assigning partial credit on borderline tokens.

\paragraph{Distillation Coefficient $\lambda$.} Figure~\ref{fig:ablation_lambda} sweeps the distillation weight $\lambda_{\methodname} \in \{0.001, 0.01, 0.1\}$, revealing that $\lambda = 0.01$ provides an optimal, steady complementary signal without interfering with the primary RL objective. When $\lambda$ is increased to $0.1$, the distillation gradient overwhelmingly dominates the policy update; since the teacher is on average \emph{no confident} than the student in multi-turn settings (as evidenced by the negative gap in Figure~\ref{fig:7b_alfworld_gap_gate}), this over-weighted term forces the student toward inferior behaviors, causing a severe performance decline that overshadows the GRPO reward signal. Conversely, $\lambda = 0.001$ exerts insufficient corrective pressure to meaningfully aid the RL process, confirming the necessity of a carefully calibrated, moderate coefficient.

\paragraph{Distillation Objective.} Figure~\ref{fig:ablation_loss} compares three token-level matching objectives on Qwen2.5-7B: reverse KL (our default), forward KL, and Jensen--Shannon divergence (JSD), where JSD is defined as the symmetrized average with respect to the mixture $M_t = \tfrac{1}{2}\bigl(\pi_{\theta}(\cdot\mid s_t) + \pi_T(\cdot\mid s_t^{+})\bigr)$: $$ D_{\mathrm{JSD}}^{(t)} = \tfrac{1}{2}\,D_{\mathrm{KL}}\!\bigl(\pi_{\theta}(\cdot\mid s_t)\,\|\,M_t\bigr) + \tfrac{1}{2}\,D_{\mathrm{KL}}\!\bigl(\pi_T(\cdot\mid s_t^{+})\,\|\,M_t\bigr). $$ Reverse KL clearly outperforms both alternatives, aligning perfectly with our design rationale in Section~\ref{sec:opsd}: the reverse direction $D_{\mathrm{KL}}(\pi_{\theta}\|\pi_T)$ is inherently \emph{mode-seeking}~\citep{murphy2012machine}, encouraging the student to concentrate probability mass only on modes supported by the teacher. In our partial "weak" teacher signals---where the teacher is frequently lost---this selectivity is paramount, as reverse KL naturally down-weights tokens with low teacher probability, thereby seamlessly complementing the explicit gating mechanism. In contrast, the \emph{mode-covering} nature of forward KL forces the student to spread mass across all teacher-supported tokens, indiscriminately incorporating unreliable guidance, while JSD acts as a symmetric compromise that inherits this detrimental mode-covering tendency, ultimately yielding intermediate performance.

\section{Related Work}

\subsection{Agentic RL}

Recent advances in reinforcement learning for LLMs have demonstrated strong effectiveness on verifiable reasoning tasksn~\citep{shao2024deepseekmath,yu2025dapo,guo2025ds-r1,yao2026coba,chen2026learning}. 
Building on this progress, LLMs are increasingly extended from static reasoning problems to autonomous agents that operate in dynamic, open-world environments, including GUI automation~\citep{ye2025mobileagentv3}, gameplay~\citep{shridhar2020alfworld}, and embodied control~\citep{wang2023voyager}. 
In these settings, agents must make sequential decisions based on environment observations and feedback, making agentic RL a crucial post-training recipe for improving their decision-making capabilities~\citep{lu2025uis1,dong2025arpo,feng2025gigpo,lu2026uir1,lu2026uicopilot,shi2026skill1}.

\subsection{OPSD}
  On-policy distillation (OPD) supervises a student on its own generated sequences, avoiding offline distribution mismatch~\citep{agarwal2024gkd,gu2026minillm}. GKD-style methods~\citep{agarwal2024gkd,wen2023fdivergence} minimize token-level divergences but require full-vocabulary teacher distributions, while PG-style methods~\citep{yang2026rlsd,xu2026tip} convert discrepancy into token-level rewards but risk high-variance updates. For multi-turn agents, TCOD~\citep{wang2026tcod} applies a turn-level curriculum to mitigate compounding drift, but relies on rigid schedules. On-Policy Self-Distillation (OPSD)~\citep{zhao2026opsd,he2026sdzero} further removes the need for a separate teacher by conditioning only on privileged context.
\paragraph{Hybrid Methods} Recent works have explored combining RL with distillation to leverage their complementary strengths~\citep{wang2026skillsd,yang2026rlsd,ding2026hdpo}, but suffer from rigid hand-crafted scheduling or substantially unstable updates. In contrast, our method treats distillation as a strictly separate auxiliary objective with adaptive, bounded, token-level gating, preserving the unbiasedness of the RL advantage while selectively injecting only beneficial teacher signals.

  \section{Conclusion}
  We presented \methodname{}, which reconciles RL and OPSD for multi-turn agent training through a sigmoid gate that lets each    
  token autonomously regulate its distillation intensity. This preserves RL as the unbiased optimization backbone while           
  selectively extracting beneficial teacher signals. Experiments across three benchmarks and three model scales confirm consistent
   gains over both pure RL and hybrid baselines.                                                                                 
\newpage
\bibliography{colm2026_conference}

@misc{wang2026tcod,
      title={TCOD: Exploring Temporal Curriculum in On-Policy Distillation for Multi-turn Autonomous Agents}, 
      author={Jiaqi Wang and Wenhao Zhang and Weijie Shi and Yaliang Li and James Cheng},
      year={2026},
      eprint={2604.24005},
      archivePrefix={arXiv},
      primaryClass={cs.LG},
      url={https://arxiv.org/abs/2604.24005}, 
}

@misc{yang2026rlsd,
      title={Self-Distilled RLVR}, 
      author={Chenxu Yang and Chuanyu Qin and Qingyi Si and Minghui Chen and Naibin Gu and Dingyu Yao and Zheng Lin and Weiping Wang and Jiaqi Wang and Nan Duan},
      year={2026},
      eprint={2604.03128},
      archivePrefix={arXiv},
      primaryClass={cs.LG},
      url={https://arxiv.org/abs/2604.03128}, 
}

@misc{xia2026skillrl,
      title={SkillRL: Evolving Agents via Recursive Skill-Augmented Reinforcement Learning}, 
      author={Peng Xia and Jianwen Chen and Hanyang Wang and Jiaqi Liu and Kaide Zeng and Yu Wang and Siwei Han and Yiyang Zhou and Xujiang Zhao and Haifeng Chen and Zeyu Zheng and Cihang Xie and Huaxiu Yao},
      year={2026},
      eprint={2602.08234},
      archivePrefix={arXiv},
      primaryClass={cs.LG},
      url={https://arxiv.org/abs/2602.08234}, 
}

@misc{wang2026skillsd,
      title={Skill-SD: Skill-Conditioned Self-Distillation for Multi-turn LLM Agents}, 
      author={Hao Wang and Guozhi Wang and Han Xiao and Yufeng Zhou and Yue Pan and Jichao Wang and Ke Xu and Yafei Wen and Xiaohu Ruan and Xiaoxin Chen and Honggang Qi},
      year={2026},
      eprint={2604.10674},
      archivePrefix={arXiv},
      primaryClass={cs.LG},
      url={https://arxiv.org/abs/2604.10674}, 
}

@misc{lu2026skill0,
      title={SKILL0: In-Context Agentic Reinforcement Learning for Skill Internalization}, 
      author={Zhengxi Lu and Zhiyuan Yao and Jinyang Wu and Chengcheng Han and Qi Gu and Xunliang Cai and Weiming Lu and Jun Xiao and Yueting Zhuang and Yongliang Shen},
      year={2026},
      eprint={2604.02268},
      archivePrefix={arXiv},
      primaryClass={cs.LG},
      url={https://arxiv.org/abs/2604.02268}, 
}

@misc{xu2026tip,
      title={TIP: Token Importance in On-Policy Distillation}, 
      author={Yuanda Xu and Hejian Sang and Zhengze Zhou and Ran He and Zhipeng Wang and Alborz Geramifard},
      year={2026},
      eprint={2604.14084},
      archivePrefix={arXiv},
      primaryClass={cs.LG},
      url={https://arxiv.org/abs/2604.14084}, 
}

@article{yao2022webshop,
  title={Webshop: Towards scalable real-world web interaction with grounded language agents},
  author={Yao, Shunyu and Chen, Howard and Yang, John and Narasimhan, Karthik},
  journal={Advances in Neural Information Processing Systems},
  volume={35},
  pages={20744--20757},
  year={2022}
}

@article{shridhar2020alfworld,
  title={Alfworld: Aligning text and embodied environments for interactive learning},
  author={Shridhar, Mohit and Yuan, Xingdi and C{\^o}t{\'e}, Marc-Alexandre and Bisk, Yonatan and Trischler, Adam and Hausknecht, Matthew},
  journal={arXiv preprint arXiv:2010.03768},
  year={2020}
}

@article{jin2025searchr1,
  title={Search-r1: Training llms to reason and leverage search engines with reinforcement learning},
  author={Jin, Bowen and Zeng, Hansi and Yue, Zhenrui and Yoon, Jinsung and Arik, Sercan and Wang, Dong and Zamani, Hamed and Han, Jiawei},
  journal={arXiv preprint arXiv:2503.09516},
  year={2025}
}

@article{kwiatkowski2019nq,
  title={Natural questions: a benchmark for question answering research},
  author={Kwiatkowski, Tom and Palomaki, Jennimaria and Redfield, Olivia and Collins, Michael and Parikh, Ankur and Alberti, Chris and Epstein, Danielle and Polosukhin, Illia and Devlin, Jacob and Lee, Kenton and others},
  journal={Transactions of the Association for Computational Linguistics},
  volume={7},
  pages={453--466},
  year={2019},
  publisher={MIT Press One Rogers Street, Cambridge, MA 02142-1209, USA journals-info~…}
}

@inproceedings{joshi2017triviaqa,
  title={Triviaqa: A large scale distantly supervised challenge dataset for reading comprehension},
  author={Joshi, Mandar and Choi, Eunsol and Weld, Daniel S and Zettlemoyer, Luke},
  booktitle={Proceedings of the 55th Annual Meeting of the Association for Computational Linguistics (Volume 1: Long Papers)},
  pages={1601--1611},
  year={2017}
}

@inproceedings{mallen2023popqa,
  title={When not to trust language models: Investigating effectiveness of parametric and non-parametric memories},
  author={Mallen, Alex and Asai, Akari and Zhong, Victor and Das, Rajarshi and Khashabi, Daniel and Hajishirzi, Hannaneh},
  booktitle={Proceedings of the 61st annual meeting of the association for computational linguistics (volume 1: Long papers)},
  pages={9802--9822},
  year={2023}
}

@inproceedings{yang2018hotpotqa,
  title={HotpotQA: A dataset for diverse, explainable multi-hop question answering},
  author={Yang, Zhilin and Qi, Peng and Zhang, Saizheng and Bengio, Yoshua and Cohen, William and Salakhutdinov, Ruslan and Manning, Christopher D},
  booktitle={Proceedings of the 2018 conference on empirical methods in natural language processing},
  pages={2369--2380},
  year={2018}
}

@inproceedings{ho20202wiki,
  title={Constructing a multi-hop qa dataset for comprehensive evaluation of reasoning steps},
  author={Ho, Xanh and Nguyen, Anh-Khoa Duong and Sugawara, Saku and Aizawa, Akiko},
  booktitle={Proceedings of the 28th International Conference on Computational Linguistics},
  pages={6609--6625},
  year={2020}
}

@article{trivedi2022musique,
  title={MuSiQue: Multi-hop Questions via Single-hop Question Composition},
  author={Trivedi, Harsh and Balasubramanian, Niranjan and Khot, Tushar and Sabharwal, Ashish},
  journal={Transactions of the Association for Computational Linguistics},
  volume={10},
  pages={539--554},
  year={2022}
}

@inproceedings{press2023bamboogle,
  title={Measuring and narrowing the compositionality gap in language models},
  author={Press, Ofir and Zhang, Muru and Min, Sewon and Schmidt, Ludwig and Smith, Noah A and Lewis, Mike},
  booktitle={Findings of the Association for Computational Linguistics: EMNLP 2023},
  pages={5687--5711},
  year={2023}
}

@article{feng2025gigpo,
  title={Group-in-group policy optimization for llm agent training},
  author={Feng, Lang and Xue, Zhenghai and Liu, Tingcong and An, Bo},
  journal={arXiv preprint arXiv:2505.10978},
  year={2025}
}

@article{wang2022e5,
  title={Text embeddings by weakly-supervised contrastive pre-training},
  author={Wang, Liang and Yang, Nan and Huang, Xiaolong and Jiao, Binxing and Yang, Linjun and Jiang, Daxin and Majumder, Rangan and Wei, Furu},
  journal={arXiv preprint arXiv:2212.03533},
  year={2022}
}

@article{guo2025ds-r1,
  title={Deepseek-r1: Incentivizing reasoning capability in llms via reinforcement learning},
  author={Guo, Daya and Yang, Dejian and Zhang, Haowei and Song, Junxiao and Wang, Peiyi and Zhu, Qihao and Xu, Runxin and Zhang, Ruoyu and Ma, Shirong and Bi, Xiao and others},
  journal={arXiv preprint arXiv:2501.12948},
  year={2025}
}

@article{shao2024deepseekmath,
  title={Deepseekmath: Pushing the limits of mathematical reasoning in open language models},
  author={Shao, Zhihong and Wang, Peiyi and Zhu, Qihao and Xu, Runxin and Song, Junxiao and Bi, Xiao and Zhang, Haowei and Zhang, Mingchuan and Li, YK and Wu, Yang and others},
  journal={arXiv preprint arXiv:2402.03300},
  year={2024}
}

@misc{zhao2026opsd,
      title={Self-Distilled Reasoner: On-Policy Self-Distillation for Large Language Models}, 
      author={Siyan Zhao and Zhihui Xie and Mengchen Liu and Jing Huang and Guan Pang and Feiyu Chen and Aditya Grover},
      year={2026},
      eprint={2601.18734},
      archivePrefix={arXiv},
      primaryClass={cs.LG},
      url={https://arxiv.org/abs/2601.18734}, 
}

@book{murphy2012machine,
  title={Machine learning: a probabilistic perspective},
  author={Murphy, Kevin P},
  year={2012}
}

@article{yang2025qwen3,
  title={Qwen3 technical report},
  author={Yang, An and Li, Anfeng and Yang, Baosong and Zhang, Beichen and Hui, Binyuan and Zheng, Bo and Yu, Bowen and Gao, Chang and Huang, Chengen and Lv, Chenxu and others},
  journal={arXiv preprint arXiv:2505.09388},
  year={2025}
}

@article{team2025kimi,
  title={Kimi k2: Open agentic intelligence},
  author={Team, Kimi and Bai, Yifan and Bao, Yiping and Charles, Y and Chen, Cheng and Chen, Guanduo and Chen, Haiting and Chen, Huarong and Chen, Jiahao and Chen, Ningxin and others},
  journal={arXiv preprint arXiv:2507.20534},
  year={2025}
}

@inproceedings{lu2026uir1,
  title={Ui-r1: Enhancing efficient action prediction of gui agents by reinforcement learning},
  author={Lu, Zhengxi and Chai, Yuxiang and Guo, Yaxuan and Yin, Xi and Liu, Liang and Wang, Hao and Xiao, Han and Ren, Shuai and Zhao, Pengxiang and Liu, Guangyi and others},
  booktitle={Proceedings of the AAAI Conference on Artificial Intelligence},
  volume={40},
  number={21},
  pages={17608--17616},
  year={2026}
}

@article{lu2025uis1,
  title={Ui-s1: Advancing gui automation via semi-online reinforcement learning},
  author={Lu, Zhengxi and Ye, Jiabo and Tang, Fei and Shen, Yongliang and Xu, Haiyang and Zheng, Ziwei and Lu, Weiming and Yan, Ming and Huang, Fei and Xiao, Jun and others},
  journal={arXiv preprint arXiv:2509.11543},
  year={2025}
}

@inproceedings{shi2025toollearning,
  title={Tool learning in the wild: Empowering language models as automatic tool agents},
  author={Shi, Zhengliang and Gao, Shen and Yan, Lingyong and Feng, Yue and Chen, Xiuyi and Chen, Zhumin and Yin, Dawei and Verberne, Suzan and Ren, Zhaochun},
  booktitle={Proceedings of the ACM on Web Conference 2025},
  pages={2222--2237},
  year={2025}
}

@article{comanici2025gemini,
  title={Gemini 2.5: Pushing the frontier with advanced reasoning, multimodality, long context, and next generation agentic capabilities},
  author={Comanici, Gheorghe and Bieber, Eric and Schaekermann, Mike and Pasupat, Ice and Sachdeva, Noveen and Dhillon, Inderjit and Blistein, Marcel and Ram, Ori and Zhang, Dan and Rosen, Evan and others},
  journal={arXiv preprint arXiv:2507.06261},
  year={2025}
}

@article{team2026longcat-2601,
  title={Longcat-flash-thinking-2601 technical report},
  author={Team, Meituan LongCat and Gui, Anchun and Li, Bei and Tao, Bingyang and Zhou, Bole and Chen, Borun and Zhang, Chao and Gao, Chen and Zhang, Chen and Han, Chengcheng and others},
  journal={arXiv preprint arXiv:2601.16725},
  year={2026}
}

@article{shen2023hugginggpt,
  title={Hugginggpt: Solving ai tasks with chatgpt and its friends in hugging face},
  author={Shen, Yongliang and Song, Kaitao and Tan, Xu and Li, Dongsheng and Lu, Weiming and Zhuang, Yueting},
  journal={Advances in Neural Information Processing Systems},
  volume={36},
  pages={38154--38180},
  year={2023}
}

@article{jimenez2023swebench,
  title={Swe-bench: Can language models resolve real-world github issues?},
  author={Jimenez, Carlos E and Yang, John and Wettig, Alexander and Yao, Shunyu and Pei, Kexin and Press, Ofir and Narasimhan, Karthik},
  journal={arXiv preprint arXiv:2310.06770},
  year={2023}
}

@article{dong2025arpo,
  title={Agentic reinforced policy optimization},
  author={Dong, Guanting and Mao, Hangyu and Ma, Kai and Bao, Licheng and Chen, Yifei and Wang, Zhongyuan and Chen, Zhongxia and Du, Jiazhen and Wang, Huiyang and Zhang, Fuzheng and others},
  journal={arXiv preprint arXiv:2507.19849},
  year={2025}
}

@misc{ye2026opcd,
      title={On-Policy Context Distillation for Language Models}, 
      author={Tianzhu Ye and Li Dong and Xun Wu and Shaohan Huang and Furu Wei},
      year={2026},
      eprint={2602.12275},
      archivePrefix={arXiv},
      primaryClass={cs.CL},
      url={https://arxiv.org/abs/2602.12275}, 
}

@misc{yang2026g-opd,
      title={Learning beyond Teacher: Generalized On-Policy Distillation with Reward Extrapolation}, 
      author={Wenkai Yang and Weijie Liu and Ruobing Xie and Kai Yang and Saiyong Yang and Yankai Lin},
      year={2026},
      eprint={2602.12125},
      archivePrefix={arXiv},
      primaryClass={cs.LG},
      url={https://arxiv.org/abs/2602.12125}, 
}

@misc{he2026sdzero,
      title={Self-Distillation Zero: Self-Revision Turns Binary Rewards into Dense Supervision}, 
      author={Yinghui He and Simran Kaur and Adithya Bhaskar and Yongjin Yang and Jiarui Liu and Narutatsu Ri and Liam Fowl and Abhishek Panigrahi and Danqi Chen and Sanjeev Arora},
      year={2026},
      eprint={2604.12002},
      archivePrefix={arXiv},
      primaryClass={cs.CL},
      url={https://arxiv.org/abs/2604.12002}, 
}

@misc{coreteam2026mimov2,
      title={MiMo-V2-Flash Technical Report}, 
      author={Core Team and Bangjun Xiao and Bingquan Xia and Bo Yang and Bofei Gao and Bowen Shen and Chen Zhang and Chenhong He and Chiheng Lou and Fuli Luo and Gang Wang and Gang Xie and Hailin Zhang and Hanglong Lv and Hanyu Li and Heyu Chen and Hongshen Xu and Houbin Zhang and Huaqiu Liu and Jiangshan Duo and Jianyu Wei and Jiebao Xiao and Jinhao Dong and Jun Shi and Junhao Hu and Kainan Bao and Kang Zhou and Lei Li and Liang Zhao and Linghao Zhang and Peidian Li and Qianli Chen and Shaohui Liu and Shihua Yu and Shijie Cao and Shimao Chen and Shouqiu Yu and Shuo Liu and Tianling Zhou and Weijiang Su and Weikun Wang and Wenhan Ma and Xiangwei Deng and Bohan Mao and Bowen Ye and Can Cai and Chenghua Wang and Chengxuan Zhu and Chong Ma and Chun Chen and Chunan Li and Dawei Zhu and Deshan Xiao and Dong Zhang and Duo Zhang and Fangyue Liu and Feiyu Yang and Fengyuan Shi and Guoan Wang and Hao Tian and Hao Wu and Heng Qu and Hongfei Yi and Hongxu An and Hongyi Guan and Xing Zhang and Yifan Song and Yihan Yan and Yihao Zhao and Yingchun Lai and Yizhao Gao and Yu Cheng and Yuanyuan Tian and Yudong Wang and Zhen Tang and Zhengju Tang and Zhengtao Wen and Zhichao Song and Zhixian Zheng and Zihan Jiang and Jian Wen and Jiarui Sun and Jiawei Li and Jinlong Xue and Jun Xia and Kai Fang and Menghang Zhu and Nuo Chen and Qian Tu and Qihao Zhang and Qiying Wang and Rang Li and Rui Ma and Shaolei Zhang and Shengfan Wang and Shicheng Li and Shuhao Gu and Shuhuai Ren and Sirui Deng and Tao Guo and Tianyang Lu and Weiji Zhuang and Weikang Zhang and Weimin Xiong and Wenshan Huang and Wenyu Yang and Xin Zhang and Xing Yong and Xu Wang and Xueyang Xie and Yilin Jiang and Yixin Yang and Yongzhe He and Yu Tu and Yuanliang Dong and Yuchen Liu and Yue Ma and Yue Yu and Yuxing Xiang and Zhaojun Huang and Zhenru Lin and Zhipeng Xu and Zhiyang Chen and Zhonghua Deng and Zihan Zhang and Zihao Yue},
      year={2026},
      eprint={2601.02780},
      archivePrefix={arXiv},
      primaryClass={cs.CL},
      url={https://arxiv.org/abs/2601.02780}, 
}

@misc{glm5team2026glm5,
      title={GLM-5: from Vibe Coding to Agentic Engineering}, 
      author={GLM-5-Team and : and Aohan Zeng and Xin Lv and Zhenyu Hou and Zhengxiao Du and Qinkai Zheng and Bin Chen and Da Yin and Chendi Ge and Chenghua Huang and Chengxing Xie and Chenzheng Zhu and Congfeng Yin and Cunxiang Wang and Gengzheng Pan and Hao Zeng and Haoke Zhang and Haoran Wang and Huilong Chen and Jiajie Zhang and Jian Jiao and Jiaqi Guo and Jingsen Wang and Jingzhao Du and Jinzhu Wu and Kedong Wang and Lei Li and Lin Fan and Lucen Zhong and Mingdao Liu and Mingming Zhao and Pengfan Du and Qian Dong and Rui Lu and Shuang-Li and Shulin Cao and Song Liu and Ting Jiang and Xiaodong Chen and Xiaohan Zhang and Xuancheng Huang and Xuezhen Dong and Yabo Xu and Yao Wei and Yifan An and Yilin Niu and Yitong Zhu and Yuanhao Wen and Yukuo Cen and Yushi Bai and Zhongpei Qiao and Zihan Wang and Zikang Wang and Zilin Zhu and Ziqiang Liu and Zixuan Li and Bojie Wang and Bosi Wen and Can Huang and Changpeng Cai and Chao Yu and Chen Li and Chengwei Hu and Chenhui Zhang and Dan Zhang and Daoyan Lin and Dayong Yang and Di Wang and Ding Ai and Erle Zhu and Fangzhou Yi and Feiyu Chen and Guohong Wen and Hailong Sun and Haisha Zhao and Haiyi Hu and Hanchen Zhang and Hanrui Liu and Hanyu Zhang and Hao Peng and Hao Tai and Haobo Zhang and He Liu and Hongwei Wang and Hongxi Yan and Hongyu Ge and Huan Liu and Huanpeng Chu and Jia'ni Zhao and Jiachen Wang and Jiajing Zhao and Jiamin Ren and Jiapeng Wang and Jiaxin Zhang and Jiayi Gui and Jiayue Zhao and Jijie Li and Jing An and Jing Li and Jingwei Yuan and Jinhua Du and Jinxin Liu and Junkai Zhi and Junwen Duan and Kaiyue Zhou and Kangjian Wei and Ke Wang and Keyun Luo and Laiqiang Zhang and Leigang Sha and Liang Xu and Lindong Wu and Lintao Ding and Lu Chen and Minghao Li and Nianyi Lin and Pan Ta and Qiang Zou and Rongjun Song and Ruiqi Yang and Shangqing Tu and Shangtong Yang and Shaoxiang Wu and Shengyan Zhang and Shijie Li and Shuang Li and Shuyi Fan and Wei Qin and Wei Tian and Weining Zhang and Wenbo Yu and Wenjie Liang and Xiang Kuang and Xiangmeng Cheng and Xiangyang Li and Xiaoquan Yan and Xiaowei Hu and Xiaoying Ling and Xing Fan and Xingye Xia and Xinyuan Zhang and Xinze Zhang and Xirui Pan and Xu Zou and Xunkai Zhang and Yadi Liu and Yandong Wu and Yanfu Li and Yidong Wang and Yifan Zhu and Yijun Tan and Yilin Zhou and Yiming Pan and Ying Zhang and Yinpei Su and Yipeng Geng and Yong Yan and Yonglin Tan and Yuean Bi and Yuhan Shen and Yuhao Yang and Yujiang Li and Yunan Liu and Yunqing Wang and Yuntao Li and Yurong Wu and Yutao Zhang and Yuxi Duan and Yuxuan Zhang and Zezhen Liu and Zhengtao Jiang and Zhenhe Yan and Zheyu Zhang and Zhixiang Wei and Zhuo Chen and Zhuoer Feng and Zijun Yao and Ziwei Chai and Ziyuan Wang and Zuzhou Zhang and Bin Xu and Minlie Huang and Hongning Wang and Juanzi Li and Yuxiao Dong and Jie Tang},
      year={2026},
      eprint={2602.15763},
      archivePrefix={arXiv},
      primaryClass={cs.LG},
      url={https://arxiv.org/abs/2602.15763}, 
}

@misc{zhang2026embarrassinglysd,
      title={Embarrassingly Simple Self-Distillation Improves Code Generation}, 
      author={Ruixiang Zhang and Richard He Bai and Huangjie Zheng and Navdeep Jaitly and Ronan Collobert and Yizhe Zhang},
      year={2026},
      eprint={2604.01193},
      archivePrefix={arXiv},
      primaryClass={cs.CL},
      url={https://arxiv.org/abs/2604.01193}, 
}

@misc{ross2011dagger,
      title={A Reduction of Imitation Learning and Structured Prediction to No-Regret Online Learning}, 
      author={Stephane Ross and Geoffrey J. Gordon and J. Andrew Bagnell},
      year={2011},
      eprint={1011.0686},
      archivePrefix={arXiv},
      primaryClass={cs.LG},
      url={https://arxiv.org/abs/1011.0686}, 
}

@misc{ding2026hdpo,
      title={HDPO: Hybrid Distillation Policy Optimization via Privileged Self-Distillation}, 
      author={Ken Ding},
      year={2026},
      eprint={2603.23871},
      archivePrefix={arXiv},
      primaryClass={cs.LG},
      url={https://arxiv.org/abs/2603.23871}, 
}

@misc{chen2019learningcheating,
      title={Learning by Cheating}, 
      author={Dian Chen and Brady Zhou and Vladlen Koltun and Philipp Krähenbühl},
      year={2019},
      eprint={1912.12294},
      archivePrefix={arXiv},
      primaryClass={cs.RO},
      url={https://arxiv.org/abs/1912.12294}, 
}

@article{ye2025mobileagentv3,
  title={Mobile-agent-v3: Fundamental agents for gui automation},
  author={Ye, Jiabo and Zhang, Xi and Xu, Haiyang and Liu, Haowei and Wang, Junyang and Zhu, Zhaoqing and Zheng, Ziwei and Gao, Feiyu and Cao, Junjie and Lu, Zhengxi and others},
  journal={arXiv preprint arXiv:2508.15144},
  year={2025}
}

@article{wang2023voyager,
  title={Voyager: An open-ended embodied agent with large language models},
  author={Wang, Guanzhi and Xie, Yuqi and Jiang, Yunfan and Mandlekar, Ajay and Xiao, Chaowei and Zhu, Yuke and Fan, Linxi and Anandkumar, Anima},
  journal={arXiv preprint arXiv:2305.16291},
  year={2023}
}

@misc{agarwal2024gkd,
      title={On-Policy Distillation of Language Models: Learning from Self-Generated Mistakes}, 
      author={Rishabh Agarwal and Nino Vieillard and Yongchao Zhou and Piotr Stanczyk and Sabela Ramos and Matthieu Geist and Olivier Bachem},
      year={2024},
      eprint={2306.13649},
      archivePrefix={arXiv},
      primaryClass={cs.LG},
      url={https://arxiv.org/abs/2306.13649}, 
}

@misc{gu2026minillm,
      title={MiniLLM: On-Policy Distillation of Large Language Models}, 
      author={Yuxian Gu and Li Dong and Furu Wei and Minlie Huang},
      year={2026},
      eprint={2306.08543},
      archivePrefix={arXiv},
      primaryClass={cs.CL},
      url={https://arxiv.org/abs/2306.08543}, 
}

@misc{wen2023fdivergence,
      title={f-Divergence Minimization for Sequence-Level Knowledge Distillation}, 
      author={Yuqiao Wen and Zichao Li and Wenyu Du and Lili Mou},
      year={2023},
      eprint={2307.15190},
      archivePrefix={arXiv},
      primaryClass={cs.CL},
      url={https://arxiv.org/abs/2307.15190}, 
}

@article{yu2025dapo,
  title={Dapo: An open-source llm reinforcement learning system at scale},
  author={Yu, Qiying and Zhang, Zheng and Zhu, Ruofei and Yuan, Yufeng and Zuo, Xiaochen and Yue, Yu and Dai, Weinan and Fan, Tiantian and Liu, Gaohong and Liu, Lingjun and others},
  journal={arXiv preprint arXiv:2503.14476},
  year={2025}
}

@misc{lu2026uicopilot,
      title={UI-Copilot: Advancing Long-Horizon GUI Automation via Tool-Integrated Policy Optimization}, 
      author={Zhengxi Lu and Fei Tang and Guangyi Liu and Kaitao Song and Xu Tan and Jin Ma and Wenqi Zhang and Weiming Lu and Jun Xiao and Yueting Zhuang and Yongliang Shen},
      year={2026},
      eprint={2604.13822},
      archivePrefix={arXiv},
      primaryClass={cs.LG},
      url={https://arxiv.org/abs/2604.13822}, 
}

@article{yao2026coba,
  title={CoBA-RL: Capability-Oriented Budget Allocation for Reinforcement Learning in LLMs},
  author={Yao, Zhiyuan and Zhang, Yi-Kai and Chen, Yuxin and Sun, Yueqing and Xu, Zishan and Yang, Yu and Hu, Tianhao and Gu, Qi and Su, Hui and Cai, Xunliang},
  journal={arXiv preprint arXiv:2602.03048},
  year={2026}
}

@misc{shi2026skill1,
      title={Skill1: Unified Evolution of Skill-Augmented Agents via Reinforcement Learning}, 
      author={Yaorui Shi and Yuxin Chen and Zhengxi Lu and Yuchun Miao and Shugui Liu and Qi GU and Xunliang Cai and Xiang Wang and An Zhang},
      year={2026},
      eprint={2605.06130},
      archivePrefix={arXiv},
      primaryClass={cs.AI},
      url={https://arxiv.org/abs/2605.06130}, 
}

@article{chen2026learning,
  title={Learning to Self-Verify Makes Language Models Better Reasoners},
  author={Chen, Yuxin and Wang, Yu and Zhang, Yi and Ye, Ziang and Cai, Zhengzhou and Shi, Yaorui and Gu, Qi and Su, Hui and Cai, Xunliang and Wang, Xiang and others},
  journal={arXiv preprint arXiv:2602.07594},
  year={2026}
}
\bibliographystyle{colm2026_conference}
\newpage
\appendix
\renewcommand{\contentsname}{Table of Contents}
\setcounter{tocdepth}{2}
\tableofcontents
\newpage
\section{Theoretical Analysis}
\label{appendix:proof}
\subsection{Design Rationale}

The central design question is how the divergence signal should enter optimization.
We adopt the reverse-KL-aligned gap
\[
\Delta_t = \log \pi_T(y_t \mid s_t^{+}) - \log \pi_{\theta}(y_t \mid s_t)
\]
rather than forward KL, because it naturally evaluates on student-sampled tokens
and avoids the computationally expensive full-vocabulary matching.
However, using this raw gap directly as a coefficient would create overly strong,
unbounded token-level gradients during early training or under severe teacher-student mismatch.
To resolve this, we wrap the gap in a sigmoid function
\[
g_t = \sigma(\beta \Delta_t),
\]
which transforms the raw discrepancy into a bounded and monotone importance weight
\[
g_t \in (0,1), \qquad \frac{\partial g_t}{\partial \Delta_t} > 0.
\]
This preserves the ordering of token importance while strictly preventing gradient explosion.
Finally, we apply a stop-gradient operator to the gate.
Detaching $g_t$ ensures it acts purely as a confidence weight
rather than creating an additional, self-referential optimization pathway,
yielding a stable, first-order weighted likelihood update.

\subsection{Theoretical Properties}

We formalize the stability and curriculum properties of \methodname{} through the following propositions.

\begin{proposition}[Equivalent Weighted Likelihood Form]
\label{prop:weighted_mle}
Assume that both $\log \pi_T(y_t \mid s_t^{+})$ and $g_t$ are detached from gradient computation.
Minimizing $\mathcal{L}_{\methodname}$ is equivalent, up to an additive constant,
to maximizing a token-weighted log-likelihood objective on student-sampled tokens:
\[
\mathcal{L}_{\methodname}
= C - \operatorname{Agg}\!\left( g_t \log \pi_{\theta}(y_t \mid s_t) \right),
\]
where
\[
C = \operatorname{Agg}\!\left( g_t \log \pi_T(y_t \mid s_t^{+}) \right)
\]
is constant with respect to $\theta$.
\end{proposition}

\begin{proof}
By definition,
\[
\mathcal{L}_{\methodname}
= \operatorname{Agg}\!\left(
g_t \bigl(\log \pi_T(y_t \mid s_t^{+}) - \log \pi_{\theta}(y_t \mid s_t)\bigr)
\right)
\]
\[
= \operatorname{Agg}\!\left(
g_t \log \pi_T(y_t \mid s_t^{+})
\right)
-
\operatorname{Agg}\!\left(
g_t \log \pi_{\theta}(y_t \mid s_t)
\right)
\]
\[
= C - \operatorname{Agg}\!\left( g_t \log \pi_{\theta}(y_t \mid s_t) \right),
\]
where the first term $C = \operatorname{Agg}\!\left( g_t \log \pi_T(y_t \mid s_t^{+}) \right)$ is constant w.r.t.\ $\theta$ since both $g_t$ and $\log \pi_T(y_t \mid s_t^{+})$ are detached.
\end{proof}

\begin{proposition}[Gradient Form]
\label{prop:grad_form}
Under the same assumptions,
the gradient of $\mathcal{L}_{\methodname}$ is strictly modulated by the bounded scalar gate:
\[
\nabla_{\theta}\mathcal{L}_{\methodname}
= - \operatorname{Agg}\!\left( g_t \nabla_{\theta}\log \pi_{\theta}(y_t \mid s_t) \right).
\]
\end{proposition}

\begin{proof}
From Proposition~\ref{prop:weighted_mle},
\[
\nabla_{\theta}\mathcal{L}_{\methodname}
= \nabla_{\theta}\left[
C - \operatorname{Agg}\!\left( g_t \log \pi_{\theta}(y_t \mid s_t) \right)
\right]
\]
\[
= 0 - \operatorname{Agg}\!\left( g_t \nabla_{\theta}\log \pi_{\theta}(y_t \mid s_t) \right)
= - \operatorname{Agg}\!\left( g_t \nabla_{\theta}\log \pi_{\theta}(y_t \mid s_t) \right).
\]
\end{proof}

\begin{proposition}[Monotonicity and Smoothness of the Gate]
\label{prop:gate}
The gate $g_t = \sigma(\beta \Delta_t)$ is strictly increasing in $\Delta_t$,
inducing an online token-level curriculum where larger discrepancies receive stronger weights.
Its derivative satisfies
\[
\frac{\partial g_t}{\partial \Delta_t}
= \beta \,\sigma(\beta \Delta_t)\bigl(1-\sigma(\beta \Delta_t)\bigr)
\in (0,\,\beta/4].
\]
\end{proposition}

\begin{proof}
By the chain rule,
\[
\frac{\partial g_t}{\partial \Delta_t}
= \beta \,\sigma'(\beta \Delta_t).
\]
Since the logistic sigmoid satisfies
\[
\sigma'(z) = \sigma(z)\bigl(1-\sigma(z)\bigr) > 0
\qquad \forall\, z \in \mathbb{R},
\]
we obtain
\[
\frac{\partial g_t}{\partial \Delta_t}
= \beta \,\sigma(\beta \Delta_t)\bigl(1-\sigma(\beta \Delta_t)\bigr)
> 0.
\]
Let $u = \sigma(\beta \Delta_t) \in (0,1)$:
\[
u(1-u) \le \left(\frac{u + (1-u)}{2}\right)^{\!2} = \frac{1}{4}
\]
\[
\frac{\partial g_t}{\partial \Delta_t}
= \beta\, u(1-u) \le \frac{\beta}{4}.
\]
\end{proof}

\begin{proposition}[Bounded Auxiliary Gradient]
\label{prop:bounded_grad}
Assume that $\|\nabla_{\theta}\log \pi_{\theta}(y_t \mid s_t)\| \le B_t$ for each valid token.
Then the gate cannot amplify the auxiliary gradient beyond the unweighted likelihood gradient:
\[
\left\| \nabla_{\theta}\mathcal{L}_{\methodname} \right\|
\le \operatorname{Agg}(B_t).
\]
\end{proposition}

\begin{proof}
By Proposition~\ref{prop:grad_form},
\[
\left\| \nabla_{\theta}\mathcal{L}_{\methodname} \right\|
= \left\| \operatorname{Agg}\!\left( g_t \nabla_{\theta}\log \pi_{\theta}(y_t \mid s_t) \right) \right\|
\]
\[
\le \operatorname{Agg}\!\left( g_t \left\| \nabla_{\theta}\log \pi_{\theta}(y_t \mid s_t) \right\| \right)
\]
\[
\le \operatorname{Agg}\!\left( 1 \cdot B_t \right)
= \operatorname{Agg}(B_t),
\]
where the first inequality is the triangle inequality and the second uses $0 < g_t < 1$ and $\|\nabla_{\theta}\log \pi_{\theta}(y_t \mid s_t)\| \le B_t$.
\end{proof}

\begin{proposition}[Effect of Not Detaching the Gate]
\label{prop:detach}
Without stop-gradient on the gate,
the non-detached token loss $\tilde{\ell}_t = \sigma(\beta \Delta_t)\,\Delta_t$
introduces an unstable self-referential coupling term into the gradient:
\[
\nabla_{\theta}\tilde{\ell}_t
= - \Bigl( g_t + \beta \Delta_t\, g_t(1-g_t) \Bigr)
\nabla_{\theta}\log \pi_{\theta}(y_t \mid s_t).
\]
\end{proposition}

\begin{proof}
Write $\tilde{\ell}_t = g_t \Delta_t$. Since $\log \pi_T(y_t \mid s_t^{+})$ is constant w.r.t.\ $\theta$,
\[
\nabla_{\theta}\Delta_t
= \nabla_{\theta}\bigl[\log \pi_T(y_t \mid s_t^{+}) - \log \pi_{\theta}(y_t \mid s_t)\bigr]
= -\nabla_{\theta}\log \pi_{\theta}(y_t \mid s_t).
\]
By the chain rule on $g_t = \sigma(\beta \Delta_t)$,
\[
\nabla_{\theta}g_t
= \beta\, \sigma'(\beta \Delta_t)\,\nabla_{\theta}\Delta_t
= \beta\, g_t(1-g_t)\,\nabla_{\theta}\Delta_t
= -\beta\, g_t(1-g_t)\,\nabla_{\theta}\log \pi_{\theta}(y_t \mid s_t).
\]
Applying the product rule,
\[
\nabla_{\theta}\tilde{\ell}_t
= (\nabla_{\theta}g_t)\,\Delta_t + g_t\,(\nabla_{\theta}\Delta_t)
\]
\[
= \bigl[-\beta\, g_t(1-g_t)\,\nabla_{\theta}\log \pi_{\theta}(y_t \mid s_t)\bigr]\,\Delta_t
  + g_t\,\bigl[-\nabla_{\theta}\log \pi_{\theta}(y_t \mid s_t)\bigr]
\]
\[
= -\beta\, g_t(1-g_t)\,\Delta_t\,\nabla_{\theta}\log \pi_{\theta}(y_t \mid s_t)
   - g_t\,\nabla_{\theta}\log \pi_{\theta}(y_t \mid s_t)
\]
\[
= -\Bigl( g_t + \beta \Delta_t\, g_t(1-g_t) \Bigr)\,
   \nabla_{\theta}\log \pi_{\theta}(y_t \mid s_t).
\]
\end{proof}

\section{Algorithm}
\label{appendix:algorithm}
The full procedure of \methodname{} is presented in Algorithm~\ref{alg:cgtd}.
We compare against five baselines listed below:
\begin{itemize}
  \item \textbf{GRPO}~\citep{shao2024deepseekmath} (Algorithm~\ref{alg:grpo}): RL baseline that optimizes the policy via a clipped surrogate objective with group-relative advantages.
  \item \textbf{OPSD}~\citep{zhao2026opsd} (Algorithm~\ref{alg:opsd}): an on-policy self-distillation method that distills token-level knowledge from a frozen reference policy $\pi_{\mathrm{ref}}$ into the student.
  \item \textbf{Skill-SD}~\citep{wang2026skillsd} (Algorithm~\ref{alg:skillsd}): a hybrid method that augments GRPO with an importance-weighted $K_3$-divergence distillation loss, using retrieved skills as privileged context to construct the teacher signal.
  \item \textbf{GRPO+OPSD} (Algorithm~\ref{alg:grpo_opsd}): a hybrid method that simply adds the OPSD distillation loss from $\pi_{\mathrm{ref}}$ as an auxiliary objective on top of GRPO training.
  \item \textbf{RLSD}~\citep{yang2026rlsd} (Algorithm~\ref{alg:rlsd}): a hybrid method that re-weights GRPO's advantages with self-teacher's gap.
\end{itemize}

\begin{algorithm}[h]
  \caption{\methodname}
  \label{alg:cgtd}                     
  \begin{algorithmic}[1]                                                                        
  \Require Policy $\pi_{\theta}$, task set $\mathcal{S}$, skill library $\mathcal{E} =          
  \{e_1,\dots,e_M\}$, group size $G$, mixing coefficient $\lambda$, sharpness $\beta$, clip     
  bound $\epsilon$                                                                              
  \For{each training iteration}                                                                 
      \State Sample a batch of tasks $\{x\}$ from $\mathcal{S}$                                 
      \For{each task $x$}                                                                       
          \State Retrieve skill $c^{+}$ from $\mathcal{E}$ \Comment{UCB / KM / Full / Random}   
          \State \mycomment{Step 1: On-policy rollout}                                          
          \State Sample $G$ responses $\{y^{(1)},\dots,y^{(G)}\} \sim \pi_{\theta}(\cdot \mid   
  x)$                                                                                           
          \State \mycomment{Step 2: Sequence-level advantage from environment}                  
          \For{$i = 1,\dots,G$}                                                                 
              \State Obtain reward $R(x, y^{(i)})$ from environment interaction                 
          \EndFor                                                                               
          \State Compute $A^{(i)} = \frac{R(x,y^{(i)}) - \mu_G}{\sigma_G}$                      
  \Comment{Group-relative advantage}                                                            
          \State \mycomment{Step 3: GRPO policy loss}                                           
          \For{$i = 1,\dots,G$}                                                                 
              \For{$t = 1,\dots,|y^{(i)}|$}                                                     
                  \State $r_t^{(i)} \gets \pi_{\theta}(y_t^{(i)} \mid s_t^{(i)}) \,/\,          
  \pi_{\theta_{\mathrm{old}}}(y_t^{(i)} \mid s_t^{(i)})$                                        
              \EndFor                                                                           
          \EndFor                                                                               
          \State Compute $\mathcal{L}_{\mathrm{GRPO}}$ via clipped surrogate with $\{A^{(i)},   
  r_t^{(i)}\}$                                                                                  
          \State \mycomment{Step 4: Token-level gated distillation}                             
          \For{$i = 1,\dots,G$}                                                                 
              \State Compute teacher logits via forward pass with $(x, c^{+}, y^{(i)})$         
              \For{$t = 1,\dots,|y^{(i)}|$}                                                     
                  \State $\Delta_t \gets \operatorname{sg}\!\bigl(\log\pi_{\theta}(y_t^{(i)}    
  \mid s_t^{+}) - \log\pi_{\theta}(y_t^{(i)} \mid s_t)\bigr)$           
                  \State $g_t \gets \sigma(\beta \cdot \Delta_t)$                                                          
                  \State $\ell_t \gets g_t \cdot \bigl(\log\pi_{\theta}(y_t^{(i)} \mid s_t^{+})
  - \log\pi_{\theta}(y_t^{(i)} \mid s_t)\bigr)$                                                 
              \EndFor
          \EndFor                                                                               
          \State $\mathcal{L}_{\methodname} \gets                                               
  \frac{1}{G}\sum_{i=1}^{G}\operatorname{Agg}\!\bigl(\ell_t^{(i)}\bigr)$                        
          \State \mycomment{Step 5: Joint policy update}                                        
          \State Update $\theta$ by minimizing $\mathcal{L}(\theta) =                           
  \mathcal{L}_{\mathrm{GRPO}}(\theta) + \lambda \cdot \mathcal{L}_{\methodname}(\theta)$        
      \EndFor                                                                                   
  \EndFor                                                                                       
  \end{algorithmic}
  \end{algorithm}

  \begin{algorithm}[t]
  \caption{GRPO}
  \label{alg:grpo}
  \begin{algorithmic}[1]
  \Require Policy $\pi_{\theta}$, task set $\mathcal{S}$, group size $G$, clip
  bounds $\epsilon_{\mathrm{lo}},\epsilon_{\mathrm{hi}}$, dual-clip constant $c$
  \For{each training iteration}
      \State Sample a batch of tasks $\{x\}$ from $\mathcal{S}$
      \For{each task $x$}
          \State \mycomment{Step 1: On-policy rollout}
          \State Sample $G$ responses $\{y^{(1)},\dots,y^{(G)}\} \sim \pi_{\theta}(\cdot \mid
  x)$
          \State \mycomment{Step 2: Sequence-level advantage from environment}
          \For{$i = 1,\dots,G$}
              \State Obtain reward $R(x, y^{(i)})$ from environment interaction
          \EndFor
          \State Compute $A^{(i)} = \frac{R(x,y^{(i)}) - \mu_G}{\sigma_G}$
  \Comment{Group-relative advantage}
          \State \mycomment{Step 3: Clipped surrogate policy loss}
          \For{$i = 1,\dots,G$}
              \For{$t = 1,\dots,|y^{(i)}|$}
                  \State $r_t^{(i)} \gets \pi_{\theta}(y_t^{(i)} \mid s_t^{(i)}) \,/\,
  \pi_{\theta_{\mathrm{old}}}(y_t^{(i)} \mid s_t^{(i)})$
                  \State $L_1 \gets -A^{(i)} r_t^{(i)}$
                  \State $L_2 \gets -A^{(i)} \operatorname{clip}(r_t^{(i)},\,
  1{-}\epsilon_{\mathrm{lo}},\, 1{+}\epsilon_{\mathrm{hi}})$
                  \State $\ell_t^{(i)} \gets \begin{cases}
  \min(-A^{(i)} c,\;\max(L_1, L_2)) & \text{if } A^{(i)} < 0 \\
  \max(L_1, L_2) & \text{otherwise}
  \end{cases}$
              \EndFor
          \EndFor
          \State $\mathcal{L}_{\mathrm{GRPO}} \gets
  \operatorname{Agg}\!\bigl(\{\ell_t^{(i)}\}\bigr)$
          \State \mycomment{Step 4: Policy update}
          \State Update $\theta$ by minimizing $\mathcal{L}(\theta) =
  \mathcal{L}_{\mathrm{GRPO}}(\theta)$
      \EndFor
  \EndFor
  \end{algorithmic}
\end{algorithm}

\begin{algorithm}[t]
  \caption{OPSD}
  \label{alg:opsd}
  \begin{algorithmic}[1]
  \Require Policy $\pi_{\theta}$, frozen reference $\pi_{\mathrm{ref}}$, task set
  $\mathcal{S}$, group size $G$, KL coefficient $\alpha$
  \For{each training iteration}
      \State Sample a batch of tasks $\{x\}$ from $\mathcal{S}$
      \For{each task $x$}
          \State \mycomment{Step 1: On-policy rollout}
          \State Sample $G$ responses $\{y^{(1)},\dots,y^{(G)}\} \sim \pi_{\theta}(\cdot \mid
  x)$
          \State \mycomment{Step 2: Token-level KL distillation from reference}
          \For{$i = 1,\dots,G$}
              \For{$t = 1,\dots,|y^{(i)}|$}
                  \State $d_t^{(i)} \gets \log\pi_{\theta}(y_t^{(i)} \mid s_t) -
  \log\pi_{\mathrm{ref}}(y_t^{(i)} \mid s_t)$
  \Comment{$D_{\mathrm{KL}}(\pi_\theta \| \pi_{\mathrm{ref}})$}
              \EndFor
          \EndFor
          \State $\mathcal{L}_{\mathrm{OPSD}} \gets
  \alpha \cdot \operatorname{Agg}\!\bigl(\{d_t^{(i)}\}\bigr)$
          \State \mycomment{Step 3: Policy update}
          \State Update $\theta$ by minimizing $\mathcal{L}(\theta) =
  \mathcal{L}_{\mathrm{OPSD}}(\theta)$
      \EndFor
  \EndFor
  \end{algorithmic}
\end{algorithm}

\begin{algorithm}[t]
  \caption{Skill-SD}
  \label{alg:skillsd}
  \begin{algorithmic}[1]
  \Require Policy $\pi_{\theta}$, task set $\mathcal{S}$, skill library $\mathcal{E} =
  \{e_1,\dots,e_M\}$, group size $G$, distillation coefficient $\lambda$, clip
  bound $\epsilon$
  \For{each training iteration}
      \State Sample a batch of tasks $\{x\}$ from $\mathcal{S}$
      \For{each task $x$}
          \State Retrieve skill $c^{+}$ from $\mathcal{E}$ \Comment{UCB}
          \State \mycomment{Step 1: On-policy rollout}
          \State Sample $G$ responses $\{y^{(1)},\dots,y^{(G)}\} \sim \pi_{\theta}(\cdot \mid
  x)$
          \State \mycomment{Step 2: Sequence-level advantage from environment}
          \For{$i = 1,\dots,G$}
              \State Obtain reward $R(x, y^{(i)})$ from environment interaction
          \EndFor
          \State Compute $A^{(i)} = \frac{R(x,y^{(i)}) - \mu_G}{\sigma_G}$
  \Comment{Group-relative advantage}
          \State \mycomment{Step 3: GRPO policy loss (same as Algorithm~\ref{alg:grpo})}
          \State Compute $\mathcal{L}_{\mathrm{GRPO}}$ via clipped surrogate with $\{A^{(i)},
  r_t^{(i)}\}$
          \State \mycomment{Step 4: Importance-weighted K3 distillation}
          \For{$i = 1,\dots,G$}
              \State Compute teacher log-probs via forward pass with $(x, c^{+}, y^{(i)})$
              \For{$t = 1,\dots,|y^{(i)}|$}
                  \State $d_t \gets \log\pi_{\theta}(y_t^{(i)} \mid s_t) -
  \log\pi_{\theta}(y_t^{(i)} \mid s_t^{+})$
  \Comment{Student $-$ Teacher}
                  \State $k_t \gets \exp(-d_t) - 1 + d_t$
  \Comment{$K_3$ divergence}
                  \State $\rho_t \gets \exp\!\bigl(\log\pi_{\theta}(y_t^{(i)} \mid s_t) -
  \log\pi_{\theta_{\mathrm{old}}}(y_t^{(i)} \mid s_t)\bigr)$
  \Comment{On-policy IS ratio}
                  \State $\ell_t^{(i)} \gets \rho_t \cdot k_t$
              \EndFor
          \EndFor
          \State $\mathcal{L}_{\text{Skill-SD}} \gets
  \operatorname{Agg}\!\bigl(\{\ell_t^{(i)}\}\bigr)$
          \State \mycomment{Step 5: Joint policy update}
          \State Update $\theta$ by minimizing $\mathcal{L}(\theta) =
  \mathcal{L}_{\mathrm{GRPO}}(\theta) + \lambda \cdot \mathcal{L}_{\text{Skill-SD}}(\theta)$
      \EndFor
  \EndFor
  \end{algorithmic}
\end{algorithm}

\begin{algorithm}[t]
  \caption{GRPO+OPSD}
  \label{alg:grpo_opsd}
  \begin{algorithmic}[1]
  \Require Policy $\pi_{\theta}$, frozen reference $\pi_{\mathrm{ref}}$, task set
  $\mathcal{S}$, group size $G$, KL coefficient $\alpha$, clip bound $\epsilon$
  \For{each training iteration}
      \State Sample a batch of tasks $\{x\}$ from $\mathcal{S}$
      \For{each task $x$}
          \State \mycomment{Step 1: On-policy rollout}
          \State Sample $G$ responses $\{y^{(1)},\dots,y^{(G)}\} \sim \pi_{\theta}(\cdot \mid
  x)$
          \State \mycomment{Step 2: Sequence-level advantage from environment}
          \For{$i = 1,\dots,G$}
              \State Obtain reward $R(x, y^{(i)})$ from environment interaction
          \EndFor
          \State Compute $A^{(i)} = \frac{R(x,y^{(i)}) - \mu_G}{\sigma_G}$
  \Comment{Group-relative advantage}
          \State \mycomment{Step 3: GRPO policy loss (same as Algorithm~\ref{alg:grpo})}
          \State Compute $\mathcal{L}_{\mathrm{GRPO}}$ via clipped surrogate with $\{A^{(i)},
  r_t^{(i)}\}$
          \State \mycomment{Step 4: Token-level KL penalty toward $\pi_{\mathrm{ref}}$}
          \For{$i = 1,\dots,G$}
              \For{$t = 1,\dots,|y^{(i)}|$}
                  \State $d_t^{(i)} \gets \log\pi_{\theta}(y_t^{(i)} \mid s_t) -
  \log\pi_{\mathrm{ref}}(y_t^{(i)} \mid s_t)$
              \EndFor
          \EndFor
          \State $\mathcal{L}_{\mathrm{OPSD}} \gets
  \alpha \cdot \operatorname{Agg}\!\bigl(\{d_t^{(i)}\}\bigr)$
          \State \mycomment{Step 5: Joint policy update}
          \State Update $\theta$ by minimizing $\mathcal{L}(\theta) =
  \mathcal{L}_{\mathrm{GRPO}}(\theta) + \mathcal{L}_{\mathrm{OPSD}}(\theta)$
      \EndFor
  \EndFor
  \end{algorithmic}
\end{algorithm}

\begin{algorithm}[t]
  \caption{RLSD}
  \label{alg:rlsd}
  \begin{algorithmic}[1]
  \Require Policy $\pi_{\theta}$, task set $\mathcal{S}$, skill library $\mathcal{E} =
  \{e_1,\dots,e_M\}$, group size $G$, mixing coefficient $\lambda$, weight clip
  bound $\epsilon_w$, policy clip bound $\epsilon$
  \For{each training iteration}
      \State Sample a batch of tasks $\{x\}$ from $\mathcal{S}$
      \For{each task $x$}
          \State Retrieve skill $c^{+}$ from $\mathcal{E}$ \Comment{UCB / KM / Full / Random}
          \State \mycomment{Step 1: On-policy rollout}
          \State Sample $G$ responses $\{y^{(1)},\dots,y^{(G)}\} \sim \pi_{\theta}(\cdot \mid
  x)$
          \State \mycomment{Step 2: Sequence-level advantage from environment}
          \For{$i = 1,\dots,G$}
              \State Obtain reward $R(x, y^{(i)})$ from environment interaction
          \EndFor
          \State Compute $A^{(i)} = \frac{R(x,y^{(i)}) - \mu_G}{\sigma_G}$
  \Comment{Group-relative advantage}
          \State \mycomment{Step 3: Token-level advantage reweighting via teacher}
          \For{$i = 1,\dots,G$}
              \State Compute teacher log-probs via forward pass with $(x, c^{+}, y^{(i)})$
              \For{$t = 1,\dots,|y^{(i)}|$}
                  \State $\delta_t \gets \log\pi_{\theta}(y_t^{(i)} \mid s_t^{+}) -
  \log\pi_{\theta_{\mathrm{old}}}(y_t^{(i)} \mid s_t)$
  \Comment{Teacher $-$ Student gap}
                  \State $w_t \gets \operatorname{clip}\!\bigl(\exp(\operatorname{sign}(A^{(i)})
  \cdot \delta_t),\;1{-}\epsilon_w,\;1{+}\epsilon_w\bigr)$
                  \State $\hat{A}_t^{(i)} \gets A^{(i)} \cdot
  \bigl[(1-\lambda) + \lambda \cdot w_t\bigr]$
              \EndFor
          \EndFor
          \State \mycomment{Step 4: Clipped surrogate with token-level advantages}
          \State Compute $\mathcal{L}_{\mathrm{RLSD}}$ via clipped surrogate with
  $\{\hat{A}_t^{(i)},\, r_t^{(i)}\}$
          \State \mycomment{Step 5: Policy update}
          \State Update $\theta$ by minimizing $\mathcal{L}(\theta) =
  \mathcal{L}_{\mathrm{RLSD}}(\theta)$
      \EndFor
  \EndFor
  \end{algorithmic}
\end{algorithm}

\section{Hyperparameters}
Table~\ref{tab:hyperparams} summarizes the method-specific hyperparameters used for all baselines and \methodname{} across our experiments. 
\begin{table}[h]
\centering
\caption{\textbf{Hyperparameters}. $\eta$: learning rate; $G$: group size; $\epsilon$: PPO clip ratio; $\lambda$: distillation loss coefficient; $\beta$: sigmoid gate sharpness; $\alpha_{\mathrm{KL}}$: KL penalty coefficient toward the reference policy; SRS: skill retrieval strategy (KM = keyword matching).}
\label{tab:hyperparams}
\begin{tabular}{l c c c c c c c}
\toprule
\textbf{Method} & $\eta$ & $G$ & $\epsilon$ & $\lambda$ & $\beta$ & $\alpha_{\mathrm{KL}}$ & SRS \\
\midrule
GRPO & $10^{-6}$ & 8 & 0.2 & --- & --- & 0.01 & --- \\
Skill-GRPO & $10^{-6}$ & 8 & 0.2 & --- & --- & 0.01 & KM \\
OPSD & $10^{-6}$ & ---   & --- & 0.01 & 5.0 & 0.01 & KM \\
Skill-SD & $10^{-6}$ & 8 & 0.2 & 0.001 & --- & 0.01 & KM \\
GRPO+OPSD & $10^{-6}$ & 8 & 0.2 & 0.01 & 0.0 & 0.01 & KM \\
RLSD & $10^{-6}$ & 8 & 0.2 & 0.5 & --- & 0.01 & KM \\
\methodname{} (Ours) & $10^{-6}$ & 8 & 0.2 & 0.01 & 5.0 & 0.01 & KM \\
\bottomrule
\end{tabular}
\end{table}
\section{Training Dynamics}
We present the full training dynamics of \methodname{} across all model scales and environments in Figures~\ref{fig:metrics_cgtd_gate_active_ratio}--\ref{fig:metrics_critic_score_mean}, tracking five diagnostic metrics throughout training.

\begin{figure}[h]
\centering
\includegraphics[width=\columnwidth]{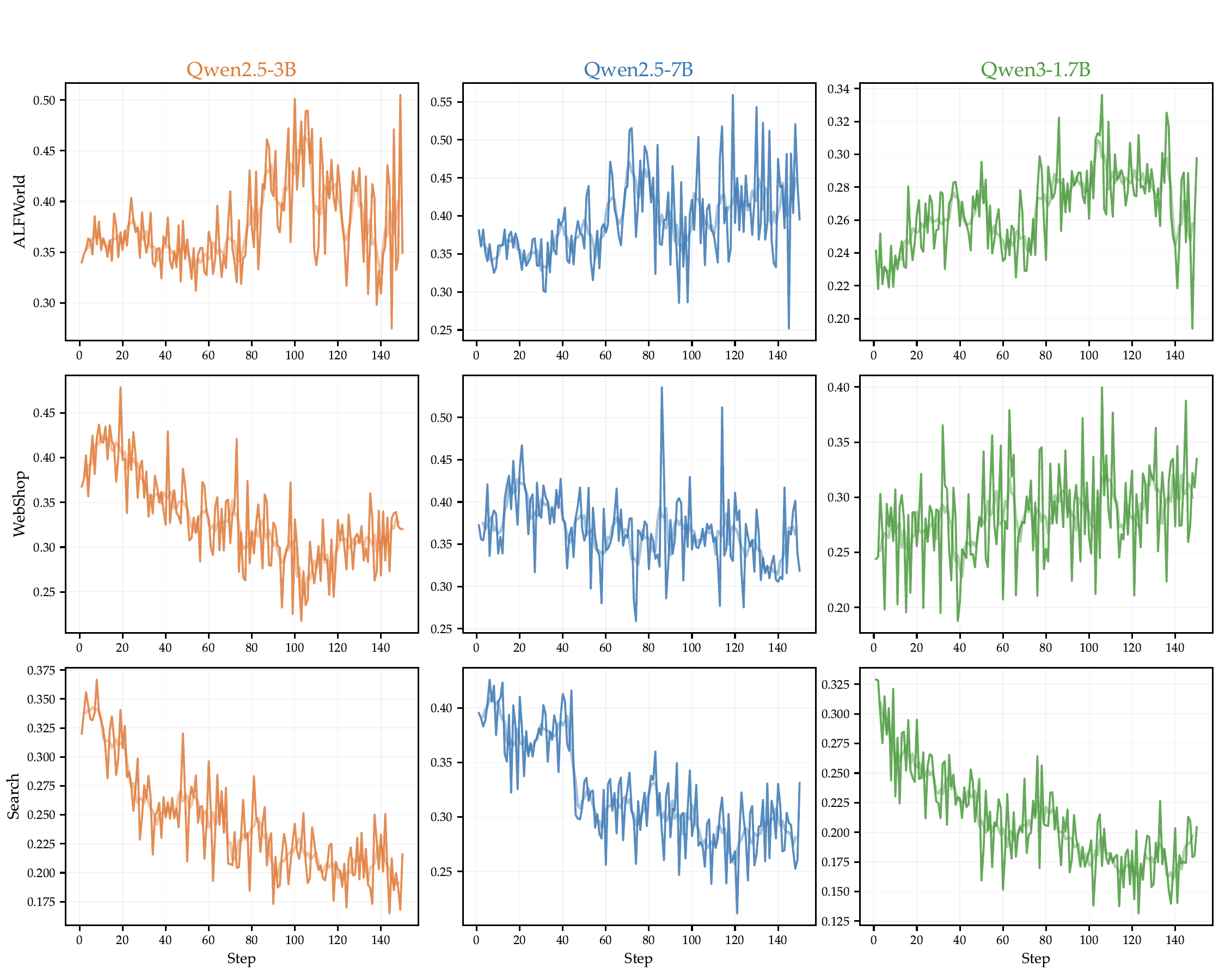}
\caption{\textbf{Gate Active Ratio} when training Qwen2.5-3B, Qwen2.5-7B and Qwen3-1.7B on ALFWorld, WebShop and Search-QA.}
\label{fig:metrics_cgtd_gate_active_ratio}
\end{figure}

\begin{figure}[h]
\centering
\includegraphics[width=\columnwidth]{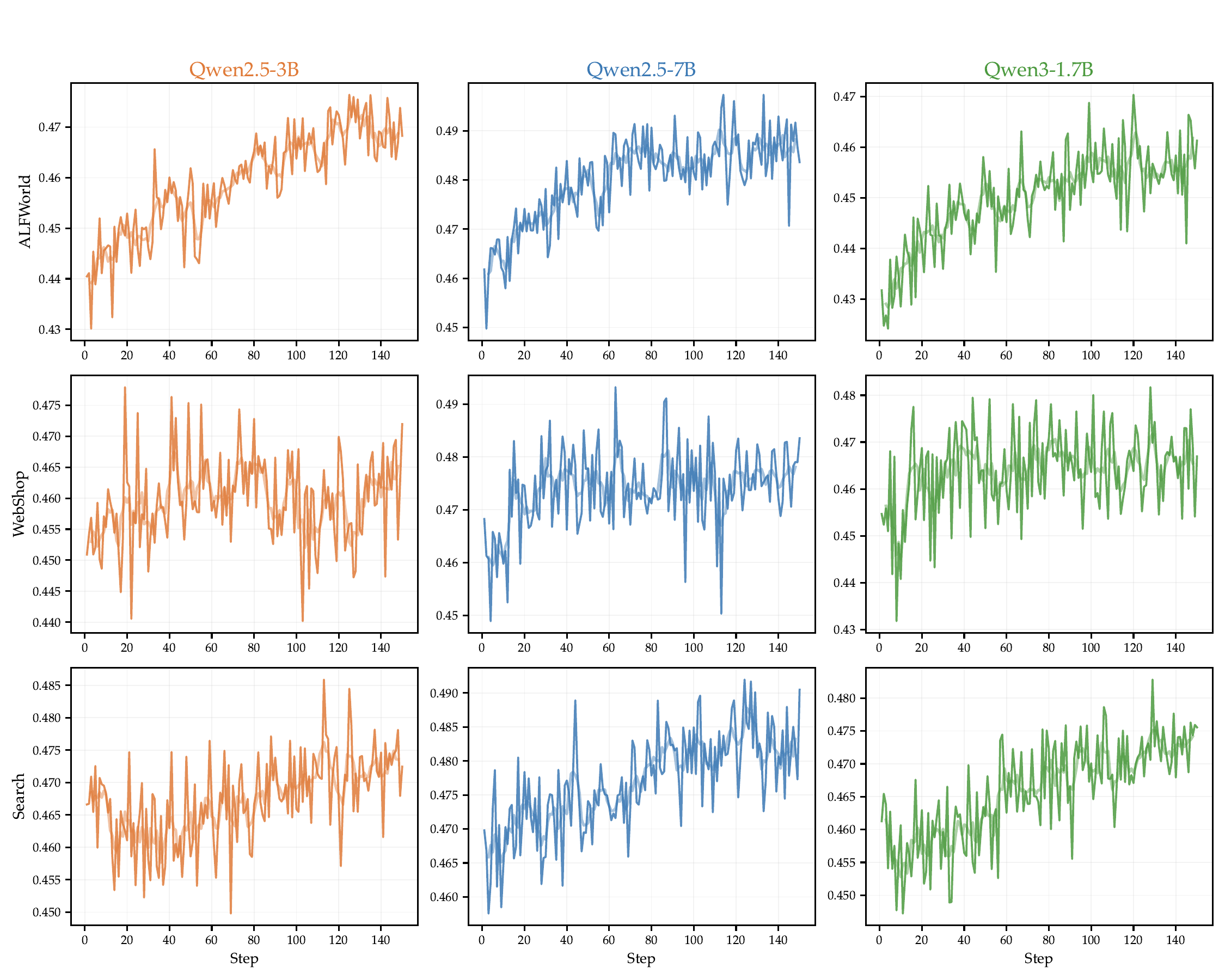}
\caption{\textbf{Gate Mean} when training Qwen2.5-3B, Qwen2.5-7B and Qwen3-1.7B on ALFWorld, WebShop and Search-QA.}
\label{fig:metrics_cgtd_gate_mean}
\end{figure}

\begin{figure}[h]
\centering
\includegraphics[width=\columnwidth]{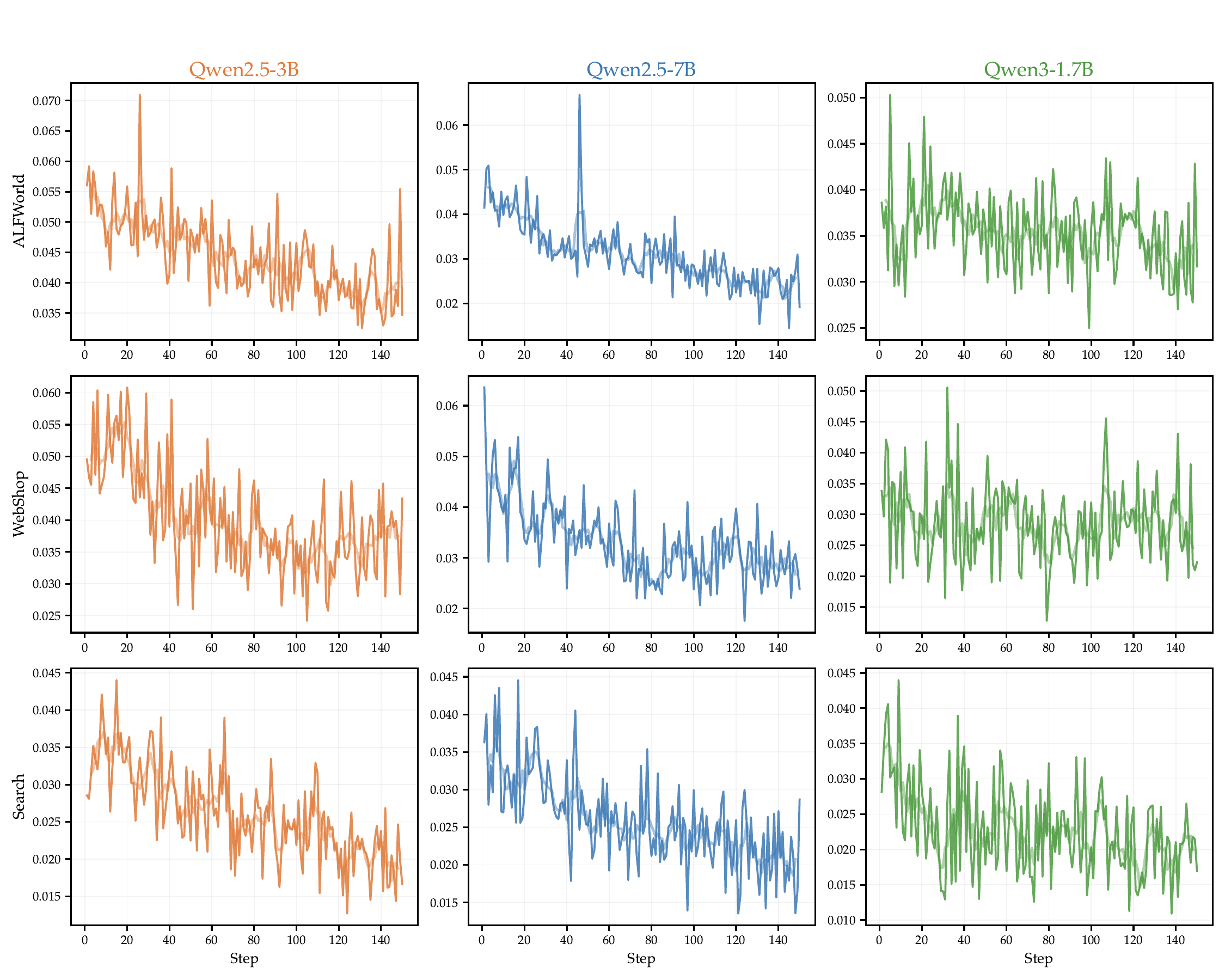}
\caption{\textbf{OPSD Loss} when training Qwen2.5-3B, Qwen2.5-7B and Qwen3-1.7B on ALFWorld, WebShop and Search-QA.}
\label{fig:metrics_cgtd_loss}
\end{figure}

\begin{figure}[h]
\centering
\includegraphics[width=\columnwidth]{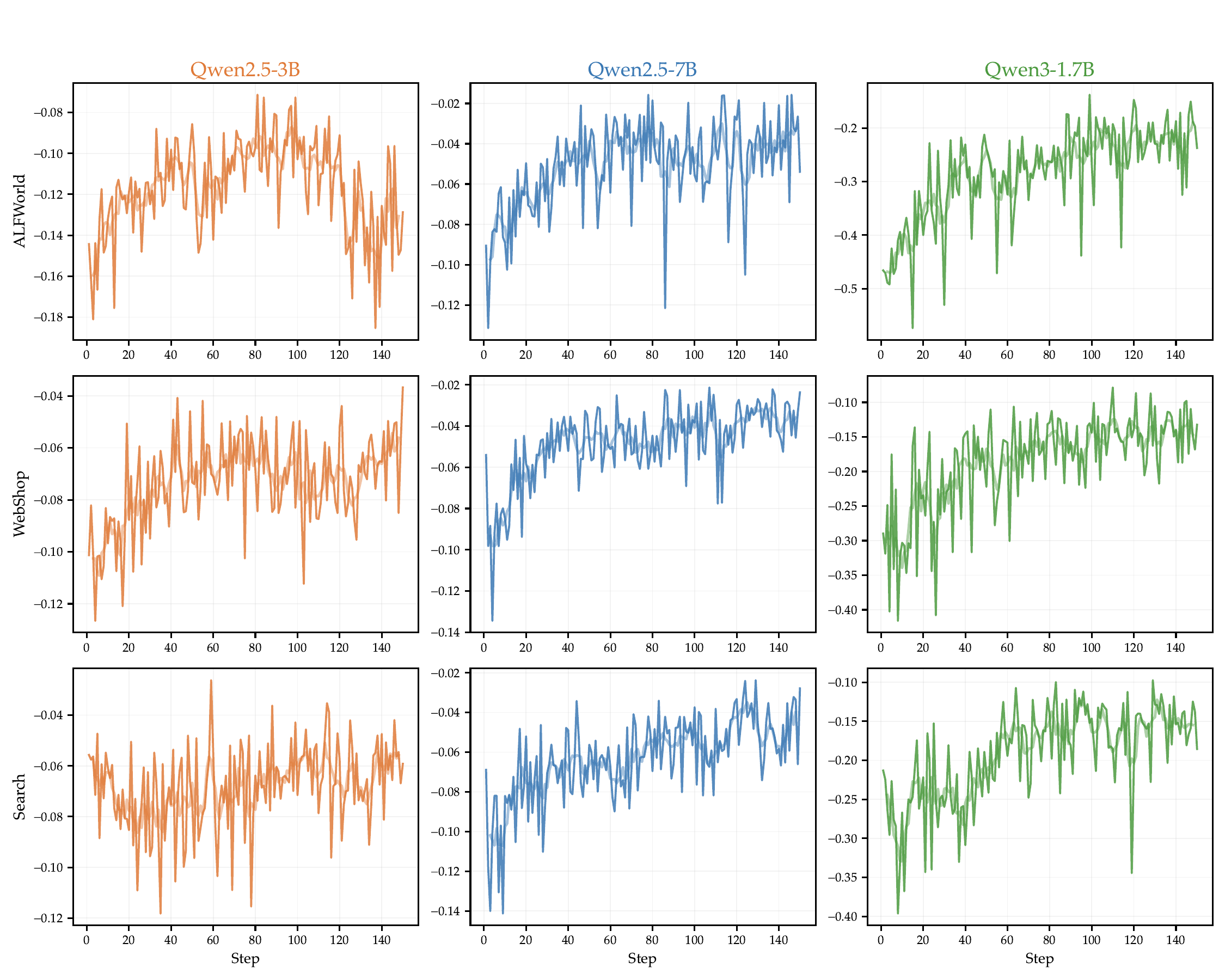}
\caption{\textbf{Teacher-Student Gap} when training Qwen2.5-3B, Qwen2.5-7B and Qwen3-1.7B on ALFWorld, WebShop and Search-QA.}
\label{fig:metrics_cgtd_teacher_gap_mean}
\end{figure}

\begin{figure}[h]
\centering
\includegraphics[width=\columnwidth]{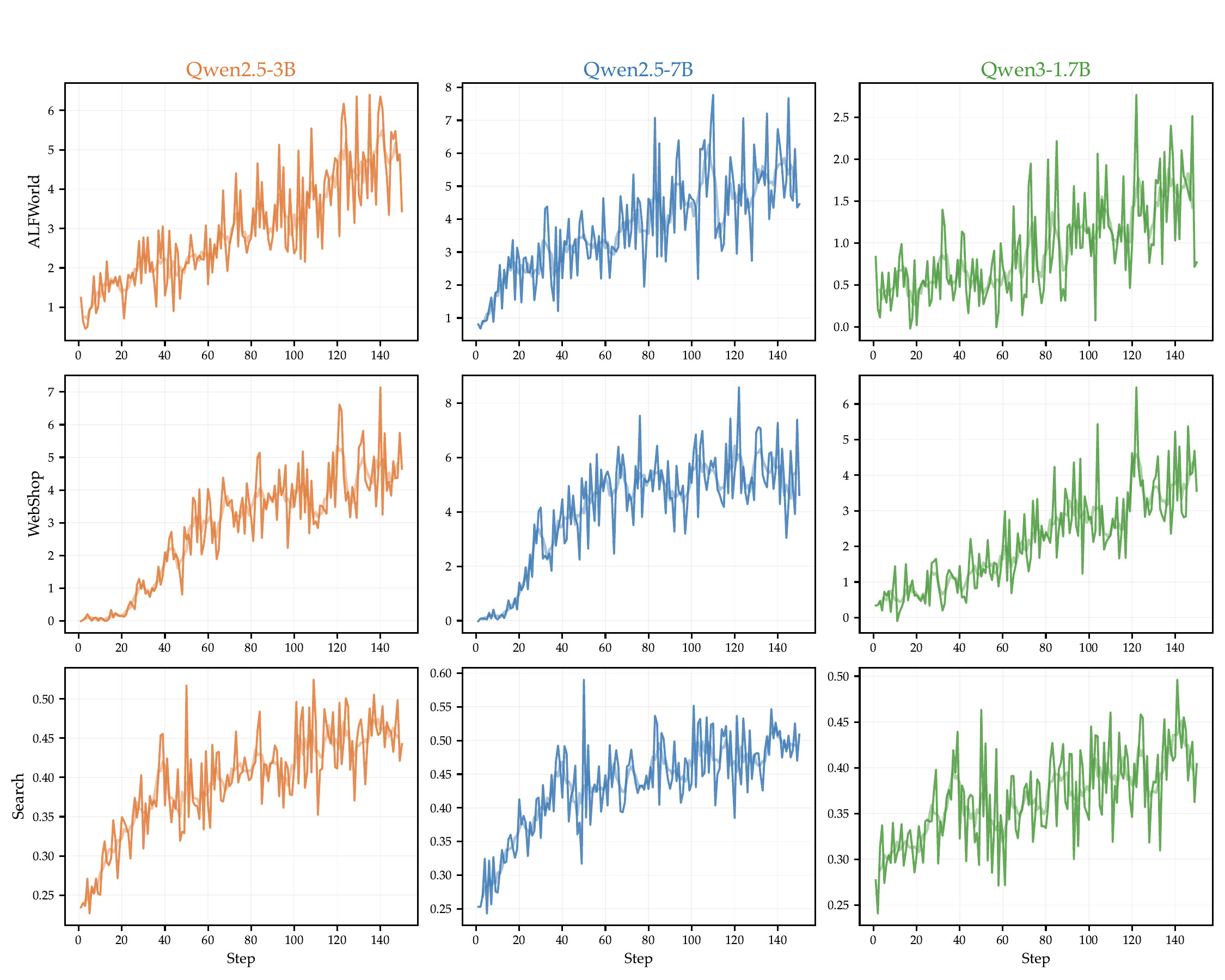}
\caption{\textbf{Reward Curve} when training Qwen2.5-3B, Qwen2.5-7B and Qwen3-1.7B on ALFWorld, WebShop and Search-QA.}
\label{fig:metrics_critic_score_mean}
\end{figure}

\newpage
\section{Prompt}

Figures~\ref{fig:prompt_alfworld}--\ref{fig:prompt_webshop} present the full prompt templates used by \methodname{} for the three evaluation environments, where \texttt{\{skill\_context\}} is populated with the retrieved skill during training and left empty at inference time. 

\begin{figure}[h]
\centering
\begin{templatebox}{Prompt of \methodname on ALFWorld}
You are an expert agent operating in the ALFRED Embodied Environment. Your task is to: \{task\_description\}.

\{skill\_context\}

Prior to this step, you have already taken \{step\_count\} step(s). Below are the most recent \{history\_length\} observations and the corresponding actions you took: \{action\_history\}

You are now at step \{current\_step\} and your current observation is: \{current\_observation\}

Your admissible actions of the current situation are: [\{admissible\_actions\}].

Now it's your turn to take an action.
You should first reason step-by-step about the current situation. This reasoning process MUST be enclosed within \texttt{<think> </think>} tags.
Once you've finished your reasoning, you should choose an admissible action for current step and present it within \texttt{<action> </action>} tags.
\end{templatebox}
\caption{Prompt template used by \methodname{} for the ALFWorld task environment.}
\label{fig:prompt_alfworld}
\end{figure}

\begin{figure}[h]
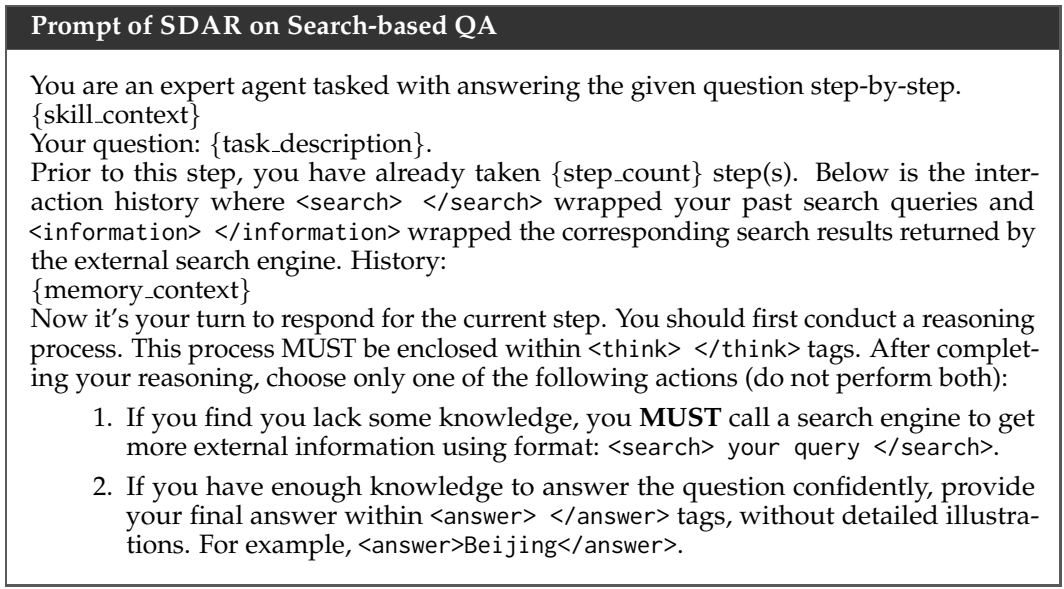

\centering
\begin{templatebox}{Prompt of \methodname on Search-based QA}
You are an expert agent tasked with answering the given question step-by-step.

\{skill\_context\}

Your question: \{task\_description\}.

Prior to this step, you have already taken \{step\_count\} step(s). Below is the interaction history where \texttt{<search> </search>} wrapped your past search queries and \texttt{<information> </information>} wrapped the corresponding search results returned by the external search engine. History:

\{memory\_context\}

Now it's your turn to respond for the current step.
You should first conduct a reasoning process. This process MUST be enclosed within \texttt{<think> </think>} tags.
After completing your reasoning, choose only one of the following actions (do not perform both):
\begin{enumerate}
    \item If you find you lack some knowledge, you \textbf{MUST} call a search engine to get more external information using format: \texttt{<search> your query </search>}.
    \item If you have enough knowledge to answer the question confidently, provide your final answer within \texttt{<answer> </answer>} tags, without detailed illustrations. For example, \texttt{<answer>Beijing</answer>}.
\end{enumerate}
\end{templatebox}
\caption{Prompt template used by \methodname{} for the Search-based QA task environment.}
\label{fig:prompt_searchqa}
\end{figure}

\begin{figure}[h]
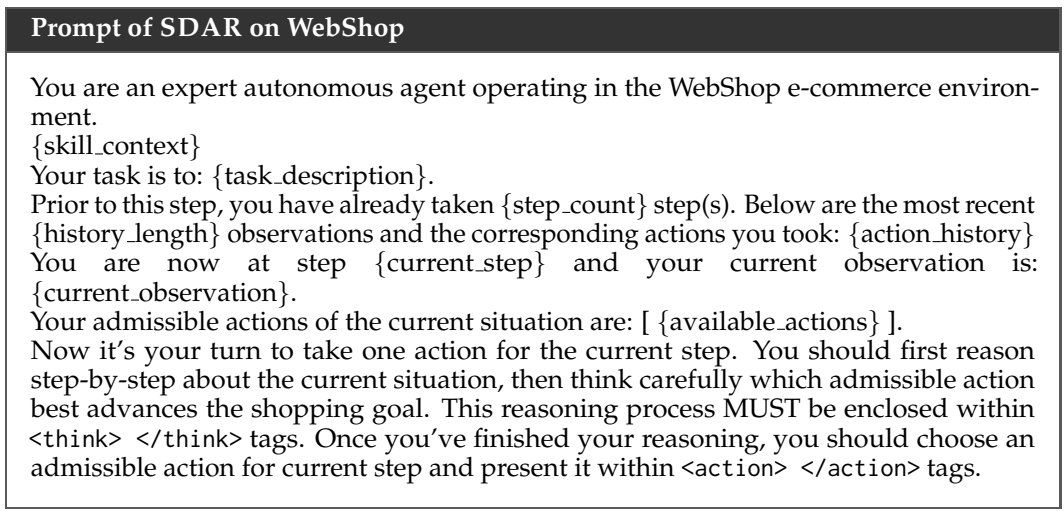

\centering
\begin{templatebox}{Prompt of \methodname on WebShop}
You are an expert autonomous agent operating in the WebShop e-commerce environment.

\{skill\_context\}

Your task is to: \{task\_description\}.

Prior to this step, you have already taken \{step\_count\} step(s). Below are the most recent \{history\_length\} observations and the corresponding actions you took: \{action\_history\}

You are now at step \{current\_step\} and your current observation is: \{current\_observation\}.

Your admissible actions of the current situation are:
[
\{available\_actions\}
].

Now it's your turn to take one action for the current step.
You should first reason step-by-step about the current situation, then think carefully which admissible action best advances the shopping goal. This reasoning process MUST be enclosed within \texttt{<think> </think>} tags.
Once you've finished your reasoning, you should choose an admissible action for current step and present it within \texttt{<action> </action>} tags.
\end{templatebox}
\caption{Prompt template used by \methodname{} for the WebShop task environment.}
\label{fig:prompt_webshop}
\end{figure}

\end{document}